\documentclass{article}

\PassOptionsToPackage{round}{natbib}

\usepackage[dblblindworkshop, final]{neurips_2026}
\workshoptitle{Attributing Model Behavior at Scale}

\usepackage[utf8]{inputenc} 
\usepackage[T1]{fontenc}    
\usepackage[pagebackref]{hyperref}
\renewcommand*{\backref}[1]{}
\renewcommand*{\backrefalt}[4]{%
  \ifcase #1\or (Cited on page~#2.)\else (Cited on pages~#2.)\fi}

\hypersetup{
    colorlinks,
    linkcolor={purple},
    citecolor={blue!50!black}
}
\usepackage{url}            
\usepackage{booktabs}       
\usepackage{amsfonts}       
\usepackage{nicefrac}       
\usepackage{microtype}      
\usepackage[table]{xcolor}  
\usepackage{arydshln}
\usepackage{url}
\usepackage{booktabs}       
\usepackage{amsfonts}       
\usepackage{nicefrac}       
\usepackage{microtype}      
\usepackage{mathtools}
\usepackage{titletoc}
\usepackage{nicefrac}
\usepackage{amsmath, amssymb, amsthm}
\usepackage{thmtools}
\usepackage{thm-restate}
\makeatletter
\newcommand{\originalonly}[1]{\ifthmt@thisistheone #1\fi}
\makeatother
\usepackage{tabularray}
\usepackage[dvipsnames]{xcolor}
\usepackage[most]{tcolorbox}
\usepackage{wrapfig}
\UseTblrLibrary{booktabs}
\usepackage[labelfont=bf]{caption}
\newtheorem{theorem}{Theorem}
\newtheorem{remark}{Remark}

\newtheorem{assumption}{Assumption}
\newtheorem{lemma}{Lemma}
\newtheorem{proposition}{Proposition}
\newtheorem{corollary}{Corollary}
\DeclareMathOperator*{\argmin}{arg\,min}
\DeclareMathOperator*{\argmax}{arg\,max}

\newtcbox{\impboxg}{on line, rounded corners, box align=base, colback=Green!10,
  colframe=white, size=fbox, arc=3pt, before upper=\strut, top=-2.5pt, bottom=-3.5pt,
  left=-2pt, right=-2pt, boxrule=0pt}
\newtcbox{\impboxr}{on line, rounded corners, box align=base, colback=Orange!10,
  colframe=white, size=fbox, arc=3pt, before upper=\strut, top=-2.5pt, bottom=-3.5pt,
  left=-2pt, right=-2pt, boxrule=0pt}
\newcommand{\impg}[1]{{\small\impboxg{\textcolor{Green!20!black}{\scriptsize #1}}}}
\newcommand{\impr}[1]{{\small\impboxr{\textcolor{Orange!20!black}{\scriptsize #1}}}}
\usepackage{xspace}
\newcommand{\method}{\textsc{LinMix}}
\newcommand{\methodx}{\textsc{LinMix}\xspace}
\newcommand{\regmix}{\textsc{RegMix}}
\newcommand{\regmixx}{\textsc{RegMix}\xspace}
\newcommand{\olmix}{\textsc{Olmix}}
\newcommand{\olmixx}{\textsc{Olmix}\xspace}
\newtcolorbox{greybox}{colback=black!6, boxrule=0pt, sharp corners,
                        before skip=4pt, after skip=4pt,
                       left=4pt, right=4pt, top=2pt, bottom=2pt, breakable}
\newenvironment{tightequation}
  {%
    \small
    \setlength{\abovedisplayskip}{4pt plus 1pt minus 2pt}
    \setlength{\belowdisplayskip}{4pt plus 1pt minus 2pt}
    \setlength{\abovedisplayshortskip}{0pt plus 1pt}
    \setlength{\belowdisplayshortskip}{2pt plus 1pt minus 1pt}
    \begin{equation}%
  }
  {\end{equation}}

\title{LinMix: Data Mixing as Linear Domain Attribution}

\author{%
    \textbf{Hubert Baniecki}\textsuperscript{1,2},\;
    \textbf{Przemyslaw Biecek}\textsuperscript{1,2},\;
    \textbf{Martin Pawelczyk}\textsuperscript{3}\\
    \textsuperscript{1}University of Warsaw\\
    \textsuperscript{2}Centre for Credible AI, Warsaw University of Technology\\
    \textsuperscript{3}Faculty of Computer Science, University of Vienna
}

\begin{document}

\maketitle

\begin{abstract}
Language model performance depends strongly on the mixture of domains in the pretraining data corpus.
Popular data mixing methods select these domain weights by training many small proxy models on random mixes and then fitting a regression model to predict performance from the mix.
We recognize this practice as a domain attribution problem, where the regression measures each domain's contribution to the target metric.
Through this lens, the standard Dirichlet sampling of proxy data mixes is a poor regression design that requires tuning and cannot support a linear fit for hundreds of domains.
Building on classical optimal experimental design, we introduce \method, which samples mixes at the vertices of the capped simplex so that each domain is at its cap or absent.
This hyperparameter-free vertex sampling decorrelates domain effects, enabling an interpretable linear model whose coefficients are the domain attributions.
The optimal mix given by \methodx provably keeps only domains with positive attribution values and requires no search over candidate mixes. 
We further establish the optimality of vertex sampling, and prove that Dirichlet sampling inflates the attribution error when its mixes concentrate around the prior data mix.
Pretraining experiments with 1B-parameter OLMo2 and 4B Qwen3.5 language models show that \methodx outperforms Dirichlet-based data mixing across reasoning, code, and math benchmarks, with relative gains of 46\% on MBPP and 14\% on GSM8K.
At 288 fine-grained domains, only vertex sampling supports a linear fit, and \methodx improves benchmark accuracy by 11\% on average over the best nonlinear baseline.
\end{abstract}

\begin{center}
\vspace{-0.32em}
\href{https://huggingface.co/collections/credible-ai/linmix}{%
  \raisebox{-0.2\height}{\includegraphics[height=1.1em]{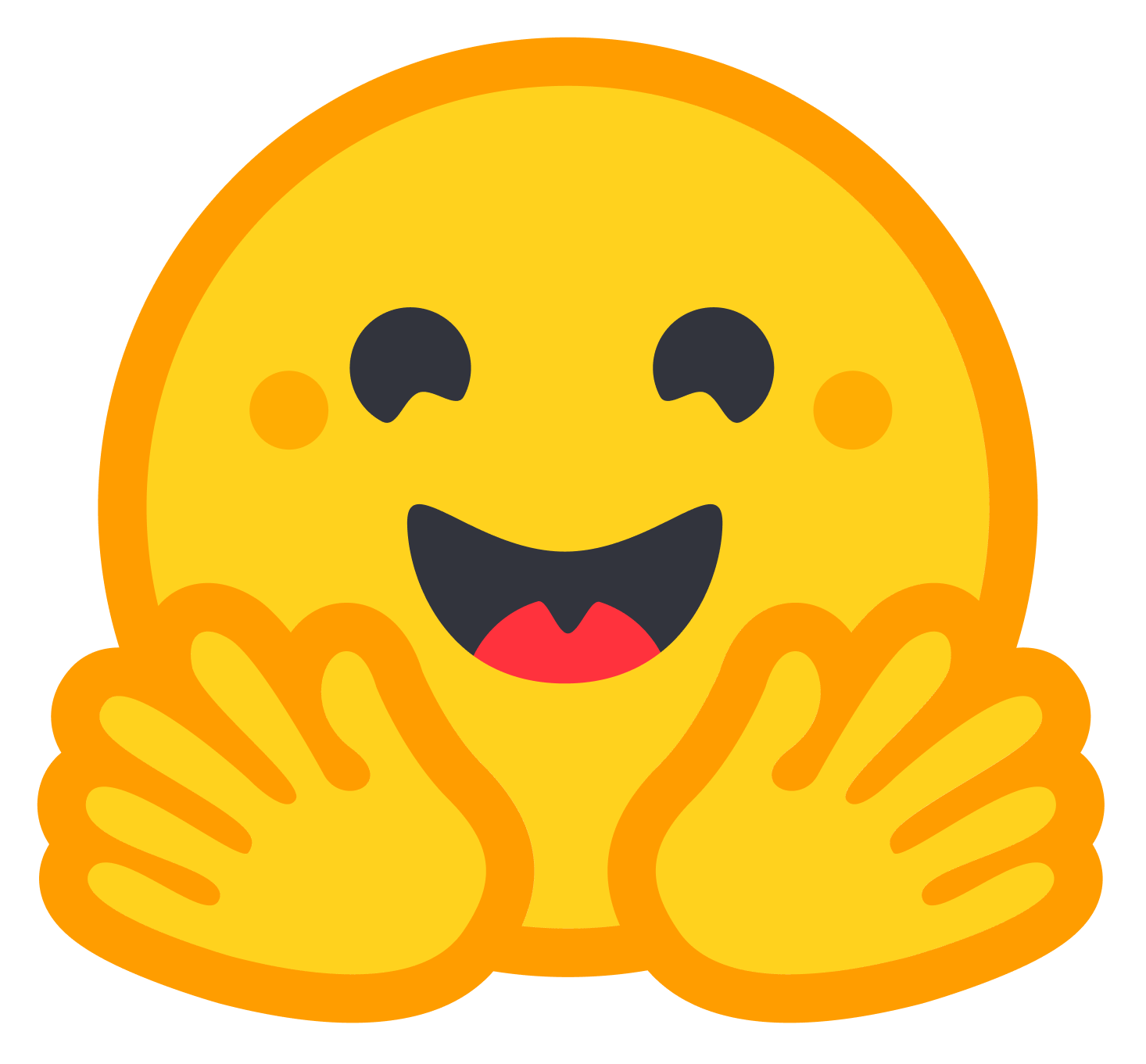}}\;\;\texttt{Models and Data}
}
\hspace{1.5em}
\href{https://github.com/credibleai/linmix}{%
  \raisebox{-0.2\height}{\includegraphics[height=1.1em]{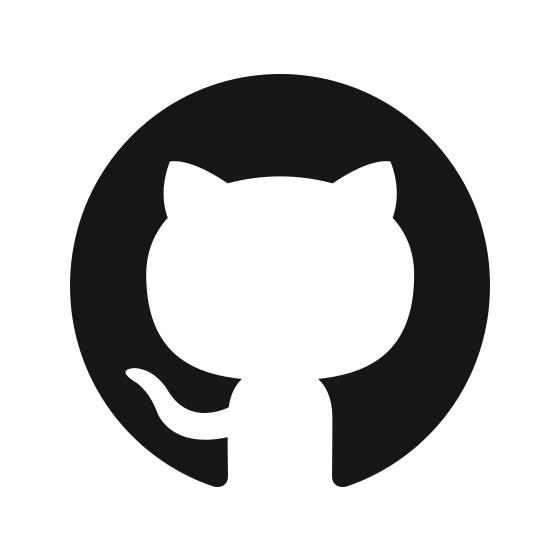}}\;\texttt{Code}%
}
\end{center}

\section{Introduction}

The composition of a pretraining dataset, drawn from an ever-growing data corpus, is one of the most decisive factors governing the downstream capabilities of large foundation models~\citep{gadre2023datacomp,olmo2026olmo3}.
Given the impressive gains in model training efficiency and benchmark performance obtained by curating an optimal dataset, \emph{data mixing} has emerged as a well-founded optimization problem.
Naturally, a source dataset can be split into the so-called \emph{domain} subsets based on their web origin~\citep[e.g., GitHub, arXiv,][]{fan2024doge}, semantic type~\citep[e.g., code, math,][]{li2024datacomp}, or statistical profile \citep[e.g., by clustering data representations,][]{diao2025nemotron}.
Data mixing is then commonly formulated as a domain-weighting problem, where each domain receives a fraction of the target model's token budget and the model is pretrained on the resulting mixture~\citep{xie2023doremi,li2024datacomp,thudi2025mixmin,xie2025chameleon,tan2026fast}.

In this work, we study \emph{regression}-based data mixing~\citep{liu2025regmix,wettig2025weborganizer,ye2025data,chen2026olmix}, which consists of two stages: 
(i)~training a swarm of small \emph{proxy models} on randomly sampled mixes, and 
(ii)~selecting the target mix by optimizing a regression model fit to predict performance from the mix.
This procedure improves both the benchmark performance and the training efficiency of language models, yet two of its design choices remain poorly understood and are addressed ad hoc in practice.
First, it is unclear from which distribution the proxy mixes should be drawn, a question we frame as the classic problem of \emph{regression design}~\citep{kiefer1959regression}, where the sampled mixes govern how much a fixed budget of proxy runs reveals about each domain's effect.
Second, we show that popular methods do not scale to hundreds of fine-grained domains, where their Dirichlet sampling and optimization hyperparameters lack working defaults and require per-setup heuristic tuning, which becomes increasingly difficult as the number of domains~increases.

\begin{figure}
    \centering
    \includegraphics[width=0.94\linewidth]{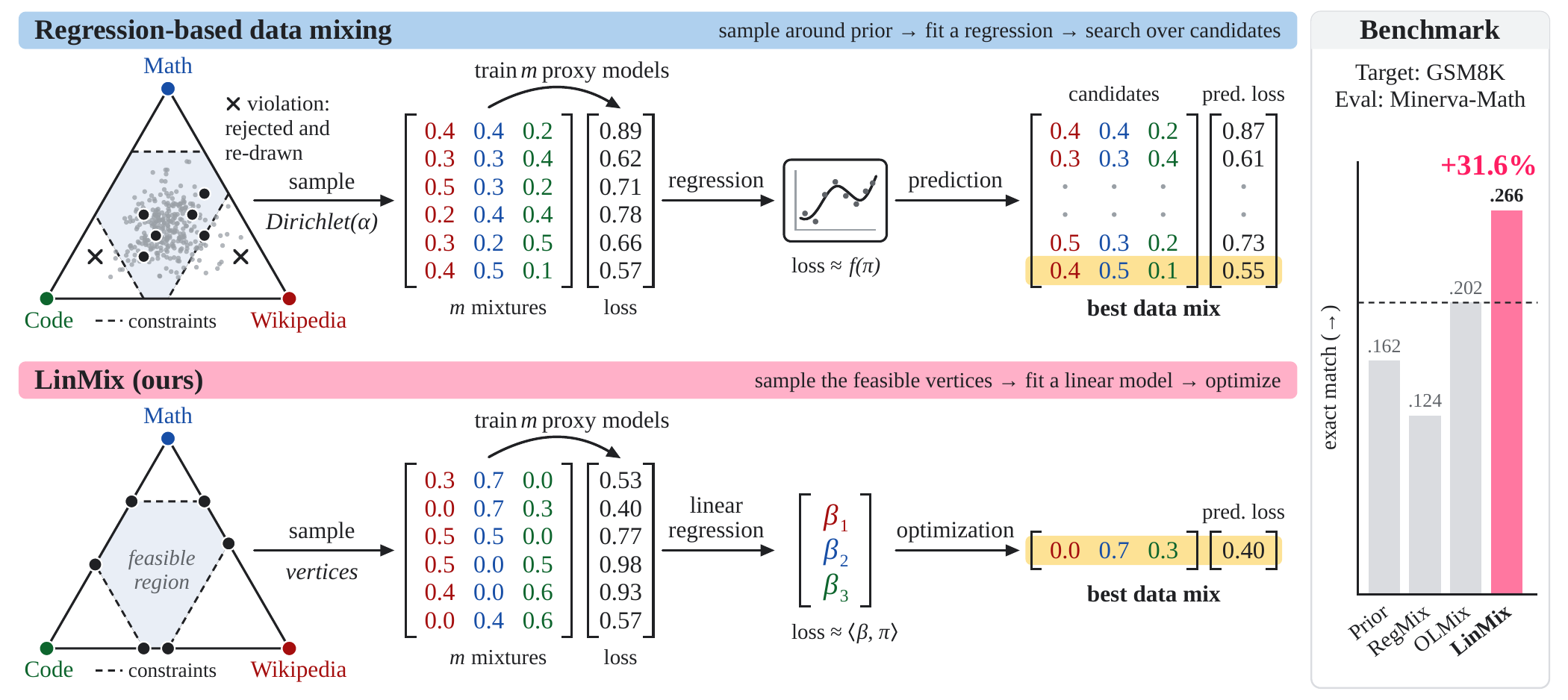}
    \caption{
    \textbf{Data mixing for language model pretraining with linear domain attribution.}
    We show that an interpretable linear regression model suffices for data mixing under vertex sampling.
    }
    \label{fig:figure1}
\end{figure}

\textbf{Domain attribution.}
In parallel, data attribution quantifies how much a model's behavior depends on its training data~\citep{koh2017understanding,ghorbani2019data}.
We formalize regression-based data mixing as a form of data attribution, in which domains replace individual training samples.
A linear fit to the proxy~runs thus yields a domain-level datamodel~\citep{ilyas2022datamodels}, whose coefficients measure each domain's marginal contribution to the target outcome.
Consequently, a reliable regression design identifies these interpretable domain weights, an objective this work addresses from first principles.

\textbf{Contributions.}
We introduce \methodx as a simple and effective data mixing method whose vertex proxy sampling scheme enables selection of the optimal mix for language model pretraining via an interpretable linear~model~(\S\ref{sec:method}).
Figure~\ref{fig:figure1} illustrates its design.
Pretraining experiments with 1B-parameter OLMo2 and 4B Qwen3.5 language models demonstrate that \methodx outperforms data mixing methods relying on Dirichlet sampling across reasoning, code, and math benchmarks~(\S\ref{sec:experiments}).
Furthermore, our theoretical analysis reveals fundamental limitations of Dirichlet-based mixing (\S\ref{sec:theoretical-analysis}). 
Beyond requiring heuristic hyperparameter tuning, it severely inflates attribution error whenever proxy mixes concentrate around the prior or when the prior becomes skewed~(Theorem~\ref{thm:dirichlet_design}).
Finally, we formally establish the theoretical optimality of vertex sampling in \methodx~(Theorem~\ref{thm:capped_vertex}).

\section{Background}

Language models are pretrained on a corpus of $n$ sequences partitioned into $d$ domains $\mathcal{D}_1, \ldots, \mathcal{D}_d$ with $n_k \coloneqq |\mathcal{D}_k|$ sequences each.
The natural proportions $(\boldsymbol{\pi}_{\text{prior}})_k \coloneqq n_k / n$ define the prior data mix, a point on the probability simplex $\Delta^{d-1}$.
Whereas training on the corpus as-is draws data according to $\boldsymbol{\pi}_{\text{prior}}$, data mixing seeks to reweight this prior into domain weights $\boldsymbol{\pi}_{\text{target}} \in \Delta^{d-1}$ that maximize the downstream performance of the \emph{target} model~\citep{xie2023doremi,fan2024doge,albalak2024survey}.
However, evaluating even a single candidate mix at the target scale would require a full pretraining run over $b_{\text{target}}$ sequences.
Instead, data mixing techniques train a swarm of $m$ small \emph{proxy} models, each on $b_{\text{proxy}} \ll b_{\text{target}}$ sequences drawn according to a candidate mix $\boldsymbol{\pi}_i$, and use their performance to predict $\boldsymbol{\pi}_{\text{target}}$; the candidate mixes form the rows of the design matrix $\boldsymbol{\Pi} \in \mathbb{R}^{m \times d}$.

\textbf{Sampling $\boldsymbol{\Pi}$.}
State-of-the-art approaches such as \regmixx~\citep{liu2025regmix} and \olmixx~\citep{chen2026olmix} populate the rows of $\boldsymbol{\Pi}$ via Dirichlet sampling, drawing mixes from a tempered version of the prior for each $i = 1, \ldots, m$:
\begin{equation}
\boldsymbol{\alpha}_i \coloneqq \gamma_i\,\boldsymbol{\pi}_{\text{prior}}^{1/\tau}
\big/
\|\boldsymbol{\pi}_{\text{prior}}^{1/\tau}\|_1,
\qquad
\boldsymbol{\pi}_i \sim \mathrm{Dirichlet}(\boldsymbol{\alpha}_i),
\qquad
\log \gamma_i \sim \mathrm{Unif}(\log \gamma_{\mathrm{lo}}, \log \gamma_{\mathrm{hi}}).
\label{eq:dirichlet_mixing}
\end{equation}
The scale $\gamma_i$ is the total concentration of the Dirichlet distribution, where small $\gamma_i$ produces sparse mixes dominated by a few domains, and large $\gamma_i$ tight perturbations of the tempered prior.
The temperature $\tau$ flattens the prior so that dominant domains (e.g., CommonCrawl) do not absorb most of the probability mass, with $\tau = 2$ by default~\citep{liu2025regmix,wettig2025weborganizer}.

Up-weighting a domain far beyond its prior share can only be realized at the target scale by repeating that domain's data, which brings diminishing returns beyond a few passes~\citep{muennighoff2025scaling}.
Recent methods therefore restrict mixes to the truncated simplex $\mathcal{P}_{\text{cap}} \coloneqq \{ \boldsymbol{\pi} \in \Delta^{d-1} \mid \pi_k \leq u_k \text{ for all } k \}$ with per-domain caps $u_k \coloneqq \min(1, \kappa\, (\boldsymbol{\pi}_{\text{prior}})_k)$ for an up-sampling factor $\kappa \geq 1$.
A mix at the cap makes at most $\kappa \cdot b_{\text{target}} / n$ passes over any domain's data; in practice, \citet{liu2025regmix} hardcode $\kappa = 15$, \citet{wettig2025weborganizer} derive $\kappa \approx n/b_{\text{target}}$ so that no data repeats at the target model scale, and \citet{chen2026olmix} set $\kappa = 4n/b_{\text{target}}$, i.e., at most four passes.

\textbf{Optimizing $\boldsymbol{\pi}_{\text{target}}$.}
Each proxy model $i$ is trained on its mix $\boldsymbol{\pi}_i$ and evaluated to obtain outcome $y_i$, e.g., validation loss or bits-per-byte, yielding the proxy dataset $(\boldsymbol{\Pi}, \boldsymbol{y})$ with $\boldsymbol{y} \in \mathbb{R}^m$.
Regression-based data mixing methods fit a parameterized model $f_\theta : \Delta^{d-1} \to \mathbb{R}$ to predict the outcome from the mix, and then predict the target mix by minimizing the fitted model over the truncated simplex:
\begin{equation}
\hat{\theta} \in \argmin_{\theta} \frac{1}{m} \sum_{i=1}^{m} \left( y_i - f_{\theta}(\boldsymbol{\pi}_i) \right)^2,
\qquad
\hat{\boldsymbol{\pi}}_{\text{target}} \in \argmin_{\boldsymbol{\pi} \in \mathcal{P}_{\text{cap}}} f_{\hat{\theta}}(\boldsymbol{\pi}).
\label{eq:fit_and_optimize}
\end{equation}
In practice, \regmixx solves this optimization approximately, evaluating a gradient-boosted $f_{\hat{\theta}}$ on a large pool of Dirichlet-sampled candidate mixes and averaging the Top-$K$ best-predicted ones for robustness~\citep{liu2025regmix}.
\olmixx~\citep{chen2026olmix} adopts the exponential mixing law of \citet{ye2025data} as $f_{\theta}(\boldsymbol{\pi}) = \exp\big(\sum_{j=1}^{d} \theta_j \pi_j\big) + \theta_0$, which is convex in $\boldsymbol{\pi}$ and can hence be minimized exactly over $\mathcal{P}_{\text{cap}}$.
Finally, $\mathcal{D}_{\text{target}}$ is constructed by drawing $\hat{\pi}_{\text{target},k} \cdot b_{\text{target}}$ sequences from each domain $\mathcal{D}_k$ and the target model is trained on this dataset.
Note that every step of this approach inherits its quality from the design $\boldsymbol{\Pi}$, which determines how well $\theta$ can be identified from only $m$ runs; yet how to sample $\boldsymbol{\Pi}$ has received little attention, which we address in this paper.

\section{Data Mixing as Linear Domain Attribution}
\label{sec:method}

We formalize data mixing as a \emph{domain attribution} problem seeking an interpretable linear model whose coefficients measure each domain's contribution to the target outcome.
We empirically find that such a linear fit cannot be supported using Dirichlet-sampled proxy runs, and prior work compensates by using nonlinear surrogates and per-setup hyperparameter tuning, which we report in~Section~\ref{sec:experiments}.
Instead of changing the surrogate model, we change the design by sampling $\boldsymbol{\Pi}$ to maximize the information the proxy runs carry about the domain effects, after which nonlinear models offer no better fit than the linear one.
Reliable attribution in turn enables the constrained optimization of $\boldsymbol{\pi}_{\text{target}}$, where Dirichlet-based methods degrade once $d$ reaches hundreds of fine-grained domains.

\textbf{Sampling $\boldsymbol{\Pi}$.}
Unlike prior work, we sample the $m$ mixes at the \emph{vertices} of $\mathcal{P}_{\text{cap}}$, aiming to maximize their spread measured by $\det(\boldsymbol{\Pi}^\top \boldsymbol{\Pi})$ without sampling hyperparameters.
This choice follows the D-optimality principle of experimental design~\citep{kiefer1959regression}, which we elaborate on in Appendix~\ref{app:related-work}.
For each run $i$, we draw a uniform random permutation $\sigma_i$ of domains~$\{1, \ldots, d\}$ and fill domains greedily in that order, saturating each at its cap until the unit budget is~exhausted:
\begin{align}
\pi_{i, \sigma_i(k)} \coloneqq \min\!\Big(u_{\sigma_i(k)},\; 1 - \sum_{\ell < k} \pi_{i, \sigma_i(\ell)}\Big) \quad \text{for } k = 1, \ldots, d. && (\text{Greedy fill})
\label{eq:vertex_fill}
\end{align}
Each resulting $\boldsymbol{\pi}_i$ has at most one strictly interior coordinate, with the rest pinned to $u_k$ or $0$.
Such a sparse design is near-orthogonal in expectation, rather than centered at $\boldsymbol{\pi}_{\text{prior}}$, as in Dirichlet sampling.
To further improve $\det(\boldsymbol{\Pi}^\top \boldsymbol{\Pi})$, we pair the proxy runs, letting run $2i$ reuse the permutation of run $2i-1$ in reverse, i.e. $\sigma_{2i}(k) \coloneqq \sigma_{2i-1}(d - k + 1)$, so that domains left empty by one fill are saturated first by its counterpart.
The paired fills are anti-correlated by construction, with largely disjoint supports, and thus contribute near-orthogonal directions to $\boldsymbol{\Pi}$, keeping $\boldsymbol{\Pi}^\top \boldsymbol{\Pi}$ well conditioned.

\textbf{Optimizing $\boldsymbol{\pi}_{\text{target}}$.}
Fitting the linear model $f_{\boldsymbol{\beta}}(\boldsymbol{\pi}) = \boldsymbol{\beta}^\top \boldsymbol{\pi}$ in Eq.~\eqref{eq:fit_and_optimize} by least squares yields the attribution estimate $\hat{\boldsymbol{\beta}} = (\boldsymbol{\Pi}^\top \boldsymbol{\Pi})^{-1} \boldsymbol{\Pi}^\top \boldsymbol{y}$, whose covariance scales with $(\boldsymbol{\Pi}^\top \boldsymbol{\Pi})^{-1}$.
We always center $\hat{\boldsymbol{\beta}}$ to zero mean so that $\hat{\beta}_k$ measures the marginal value of domain $k$ relative to the average.
We also negate $\boldsymbol{y}$ for target metrics where lower is better, so that a positive $\hat{\beta}_k$ indicates a domain of above-average value.
The mix selection then becomes a linear program over the truncated simplex,
\begin{equation}
\hat{\boldsymbol{\pi}}_{\text{target}} \in \argmax_{\boldsymbol{\pi} \in \mathcal{P}_{\text{cap}}} \; \langle \hat{\boldsymbol{\beta}}, \boldsymbol{\pi} \rangle,
\label{eq:linear_objective}
\end{equation}
solved by filling domains to their caps in decreasing order of $\hat{\beta}_k$ until the budget is exhausted, so the solution is itself a vertex of~$\mathcal{P}_{\text{cap}}$.
This allocation inherently treats each additional domain sequence as equally valuable, which is appropriate for most cases, especially quality-filtered data.

\textbf{Transforming the utility of $\pi_k$.}
Beyond sampling $\boldsymbol{\Pi}$ and optimizing $\boldsymbol{\pi}_{\text{target}}$, a third design choice worth studying is the utility assigned to each domain weight, which can capture the diminishing returns of repeating data from a given domain.
For example, if we assume that data value decays with domain repetition, the linear utility $\pi_k$ can be transformed to the logarithmic one $\ln \pi_k$, so that each additional sequence of a domain is worth less than the previous one.
Theoretically, these two extremes can be framed as special cases of the general isoelastic objective
\begin{equation}
\hat{\boldsymbol{\pi}}_{\text{target}} \in \argmax_{\boldsymbol{\pi} \in \mathcal{P}_{\text{cap}}} \; \sum_{k=1}^d \hat{\beta}_k \, \frac{\pi_k^{1-r} - 1}{1-r},
\qquad \text{for curvature } r \in [0,1],
\label{eq:objective_r}
\end{equation}
which recovers the linear objective in Eq.~\eqref{eq:linear_objective} as $r \to 0$ and the log-linear form $\sum_{k:\hat{\beta}_k>0} \hat{\beta}_k \ln \pi_k$ at $r = 1$.
Proposition~\ref{prop:full-support} establishes that positive domain attribution determines the support of the optimal mix, while the curvature $r$ only allocates the budget within it (see an example in App. Figure~\ref{fig:heatmap1b}).
We discuss the classical roots of this theory and its closest data selection counterpart in Appendix~\ref{app:related-work}.
\begin{greybox}
\begin{restatable}[\textbf{Domain support and uniqueness of the optimal mix}]{proposition}{isosupport}
\label{prop:full-support}
Let $\mathcal{A} \coloneqq \{k : \hat{\beta}_k > 0\}$ be the set of domains with positive estimated attribution, and assume their caps can absorb the whole budget, i.e. $u_{\mathcal{A}} \coloneqq \sum_{k \in \mathcal{A}} u_k > 1$.
At $r = 1$, the objective of Eq.~\eqref{eq:objective_r} reads as $\sum_{k \in \mathcal{A}} \hat{\beta}_k \ln \pi_k$, its log-linear form with the sum restricted to $\mathcal{A}$.
Then, for every $r \in (0,1]$, its maximizer $\boldsymbol{\pi}^\star$ is unique and satisfies $\pi^\star_k > 0$ for all $k \in \mathcal{A}$ and $\pi^\star_k = 0$ for all $k \notin \mathcal{A}$.
At $r = 0$, in contrast, a maximizer may assign zero weight to domains in $\mathcal{A}$.
\originalonly{Proofs are given in Appendix~\ref{app:proofs}.}
\end{restatable}
\end{greybox}
\begin{greybox}
\begin{restatable}[\textbf{Closed form of the optimal mix}]{corollary}{coroptimalmix}
\label{cor:optimal_mix}
Under the assumptions of Proposition~\ref{prop:full-support} and for $r \in (0,1]$, the objective of Eq.~\eqref{eq:objective_r} restricted to its positively attributed support $\mathcal{A}$ is strictly concave.
Then, the maximizer takes the closed form
$\hat{\pi}_{\text{target},k} = \min\big(u_k,\, (\hat{\beta}_k/\omega)^{1/r}\big)$ for $k \in \mathcal{A}$ and $\hat{\pi}_{\text{target},k} = 0$ otherwise,
where $\omega > 0$ is the unique scalar enforcing $\sum_k \hat{\pi}_{\text{target},k} = 1$.
\end{restatable}
\end{greybox}
As $r \to 0$, the solution in Corollary~\ref{cor:optimal_mix} collapses to the greedy fill of Eq.~\eqref{eq:vertex_fill} with domains ordered by decreasing $\hat{\beta}_k$ (for distinct $\hat{\beta}_k$).
Optimal mixes thus saturate the caps of their top domains, such as our sampled vertices, and predicting with $\hat{\boldsymbol{\pi}}_{\text{target}}$ interpolates rather than extrapolates. Unlike in Dirichlet sampling, optima here do not fall outside the convex hull of mixes clustered around the prior.

\section{Experiments}
\label{sec:experiments}

\textbf{Data.}
We use the 200B-token WebOrganizer corpus~\citep{wettig2025weborganizer} built from CommonCrawl web data and cleaned for language model pretraining.
Each document is assigned one of 24 topics (e.g., science, politics) and one of 24 formats (e.g., academic writing, news), and we split the corpus into 2048-token sequences, each drawn from a single domain.
We study two mixing granularities: 24 coarse domains given by the topics alone, and 288 fine-grained domains formed by the larger half of the $24 \times 24$ topic--format pairs, covering 94\% of the corpus (cf. App. Figure~\ref{fig:prior_mix}).
Removing the practically empty tail from the prior data mix reduces noise when estimating domain weights.

\textbf{Language models.}
Following \citet{chen2026olmix}, we build on the open source OLMo~2 implementation~\citep{walsh2025olmo2} to train 30M and 60M proxy models and 1B target models.
All models are trained past the $20$ tokens-per-parameter Chinchilla ratio~\citep{hoffmann2022empirical} to avoid undertraining~\citep{chenetal2025revisiting}, on 1B, 2B, and 32B tokens, respectively.
This setup is comparable to prior work that pairs 27M
~\citep{liu2025regmix} and 50M~\citep{wettig2025weborganizer} proxies with 1B targets.
Following \citet{wettig2025weborganizer}, we set the up-sampling cap to $\kappa \approx n / b_{\text{target}}$ so that no sequence repeats at the target scale, a non-arbitrary choice shared by all sampling methods we compare.
Hyperparameters and further training details are given in Appendix~\ref{app:experimental-setup}.

\textbf{Target benchmarks.}
We optimize data mixes for three targets: ARC~\citep[easy and challenge,][]{clark2018think} for scientific question answering; HellaSwag~\citep{zellers2019hellaswag} for commonsense reasoning; and HellaSwag combined with PIQA~\citep{bisk2020piqa} as a two-benchmark target.
In each case, we optimize the mix against proxy-model bits-per-byte~\citep[BPB,][]{gao2021pile} metric on the benchmark's training split and report the target model's accuracy on a held-out set: test for ARC, and validation for HellaSwag and PIQA, whose test labels are withheld.
We also optimize for BPB validation loss on the Common Pile~\citep[CPile,][]{kandpal2025common} data corpus to diversify the setup.

\subsection{Dirichlet sampling of proxy data mixes requires tuning hyperparameters}\label{sec:experiments-tuning}

\begin{wrapfigure}[16]{r}{0.6\textwidth}
    \vspace{-1.5em}
    \centering
    \includegraphics[width=\linewidth]{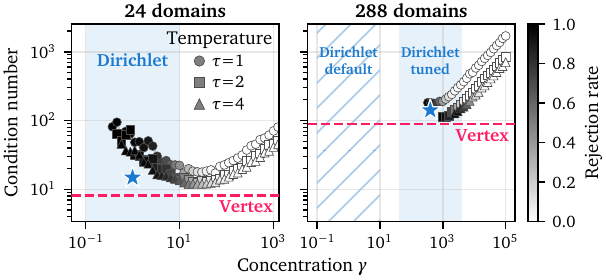}
    \caption{
    \textbf{Dirichlet sampling with default $\gamma,\tau$ hyperparameters does not converge for 288 domains.}
    Fraction of Dirichlet-sampled domain mixes that violate the target model cap constraint.
    We choose new Dirichlet hyperparameters by minimizing the condition number of matrix $\boldsymbol{\Pi}$ (blue star).
    }
    \label{fig:dirichlet_rejection_rate}
\end{wrapfigure}

Before training proxy models, we need to sample the design matrix~$\boldsymbol{\Pi}$.
Figure~\ref{fig:dirichlet_rejection_rate} shows the fraction of Dirichlet-sampled domain mixes that violate the cap constraint of $\mathcal{P}_{\text{cap}}$ and must be rejected.
For 288 domains, the baseline algorithm does not converge with the default hyperparameters $\tau = 2, \gamma_{\mathrm{lo}} = 0.1, \gamma_{\mathrm{hi}} = 5$, or even $\gamma_{\mathrm{hi}} = 10$ (cf. the hatched area in Figure~\ref{fig:dirichlet_rejection_rate}).
Since related work gives no guidance on how to tune them, we set $\tau = 2$ and propose to select $\gamma$ by minimizing the condition number of $\boldsymbol{\Pi}$, which yields mixes much closer to the prior at $\gamma_{\mathrm{lo}} = 40$ and $\gamma_{\mathrm{hi}} = 4000$.

\subsection{Vertex sampling outperforms Dirichlet, enabling a linear model fit}
\label{sec:experiments_proxy_metrics}
\begin{figure}[t]
    \centering
    \includegraphics[width=\linewidth]{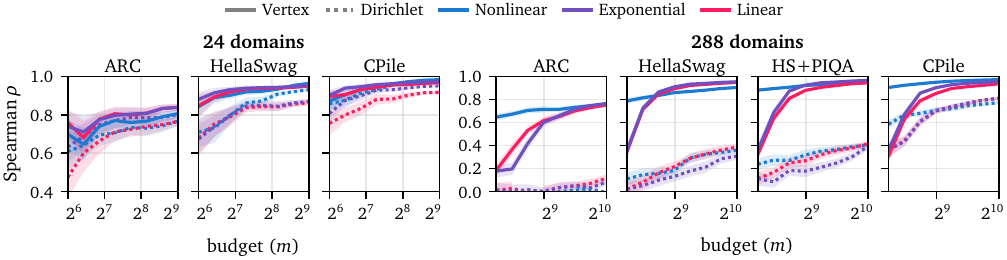}
    \caption{
    \textbf{Optimizer performance under the vertex and Dirichlet samplers.}
    Spearman correlation is averaged across 5 random repetitions of 10-fold cross-validation under a fixed proxy model budget.
    Nonlinear is implemented using LightGBM, as in \regmix, and exponential fits the exponential law to the target, as in \olmix.
    The corresponding $R^2$ values are reported in App. Figure~\ref{fig:plot_budget_extended}.
    }
    \label{fig:plot_budget_main}
\end{figure}
\vspace{-0.4em}

We train $512\times$ 30M and $1024\times$ 60M proxy models per sampler for 24 and 288 domains, respectively.
Figure~\ref{fig:plot_budget_main} shows the cross-validated fit between predicted and observed benchmark BPB (Spearman~$\rho$) under each design as the number of proxy models grows.
At 24 domains, every surrogate fits well under both samplers, but at 288 domains, no surrogate helps Dirichlet sampling even when $\gamma$ is tuned ($\rho \leq 0.39$ at $m = 1024$).
Conversely, our hyperparameter-free vertex sampling puts the interpretable linear model on par with the best nonlinear alternative (see also $R^2$ in App. Figure~\ref{fig:plot_budget_extended}).
Half of the vertex swarm already fits better than the full Dirichlet swarm, because near-prior Dirichlet mixes vary HellaSwag BPB five times less than vertex mixes, leaving little signal to regress on.
As an ablation in App. Figure~\ref{fig:scale_transfer}, we train 30M and 100M models on 288 domains, fit a linear model from the 30M swarm, and use it to rank the outcomes of 60M and 100M models at $\rho \geq 0.95$, while the Dirichlet fit does not transfer ($\rho \leq 0.42$).
We visualize the optimized domain weights in App. Figure~\ref{fig:heatmap1b}, where relaxing curvature $r$ retains the full positively attributed support~(Prop.~\ref{prop:full-support}).
Based on it, we transform the utility on 24 domains with $\ln\pi$ for HellaSwag, which adds Social Life, Fashion \& Beauty, Games, and Industrial to the four domains at their caps (Health, Home \& Hobbies, Food \& Dining, Adult); and with $\sqrt{\pi}$ for CPile, which adds a sixth Hardware domain to the data mix.

\subsection{\methodx outperforms data mixing baselines on a 1B target model benchmark}

We compare \methodx to a nonlinear \regmix, which outperforms several data selection methods and was applied in WebOrganizer, as well as \olmix, which benchmarks several data mixing approaches and finds the exponential model to be state-of-the-art.
Table~\ref{tab:performance1b} shows target benchmark performance where \methodx with a linear model significantly improves data mixing with 288 domains, and is on par or better with 24 domains.
This result is consistent with the proxy metric values reported in Section~\ref{sec:experiments_proxy_metrics}, which here translate into downstream performance.
An interpretable linear regression suffices for data mixing that targets the CPile validation loss under vertex sampling, i.e. it is on par with or better than complex Dirichlet sampling with nonlinear regression.
We ablate on two optimization factors in Appendix~\ref{app:additional-results}: Table~\ref{tab:ablation_scaling} indicates that scaling the linear model coefficients by standard error could improve performance, and Table~\ref{tab:ablation_r} demonstrates how transforming the utility of $\pi$ towards $\ln\pi$ relaxes the 24-domain support for~HellaSwag (Proposition~\ref{prop:full-support}), yielding better results.

\subsection{Scaling to a 4B Qwen3.5 model benchmarked on code and math capabilities}
\begin{wrapfigure}[33]{r}{0.5\textwidth}
    \vspace{-1.5em}
    \centering
    \includegraphics[width=\linewidth]{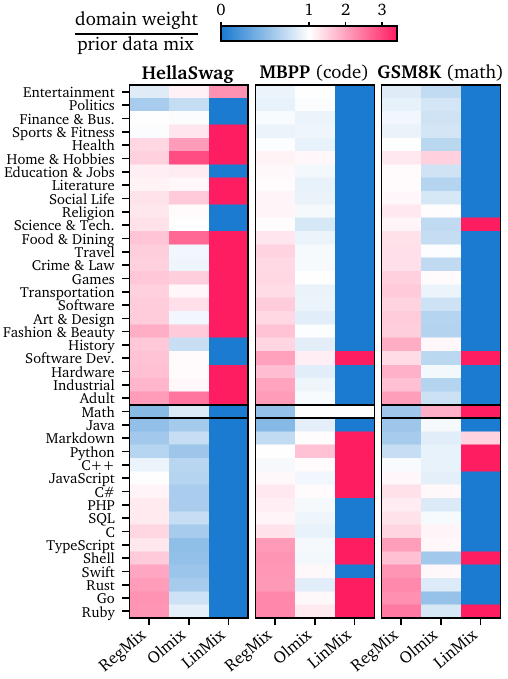}
    \caption{
    Domain weights relative to the prior mix across data mixing methods and target benchmarks.
    \methodx downweights math and code domains for language modeling (blue in HellaSwag), upweights programming languages and software domains for coding (red in MBPP), as well as math and science domains for math capabilities (red in GSM8K).
    }
    \label{fig:heatmap4b}
\end{wrapfigure}
We demonstrate the broader applicability of \methodx by doubling the data corpus and training a target model with ten times more compute.
The 24-topic WebOrganizer dataset is extended with Nemotron-CC-Math~\citep{mahabadi2026nemotron} and Stack-Edu~\citep{allal2025smollm} split into 15 programming languages, yielding a 416B-token corpus spanning 40 domains.
We pretrain 4B Qwen3.5~\citep{qwen2026qwen35} target models on 80B tokens at a $20\times$ ratio to allow for more diverse experiments.
We train $512\times$ 100M proxy models per sampler to prevent under-parametrization, given the more challenging coding and math capabilities.
Mixes are optimized towards training splits of HellaSwag, MBPP~\citep{austin2021mbpp}, and GSM8K~\citep{cobbe2021gsm8k}.
Figure~\ref{fig:heatmap4b} shows an example of optimized domain weights targeting three diverse model capabilities.
It illustrates that \regmixx obtains similar mixes across the three target capabilities, while the exponential law of \olmixx differentiates between targets, but to a lesser degree than \methodx.
Table~\ref{tab:performance4b} shows that \methodx delivers strong performance across the language, coding, and math benchmarks, transferring these capabilities to held-out benchmarks such as HumanEval~\citep{chen2021humaneval} and MATH~\citep{hendrycks2021math,lewkowycz2022minerva}.
We additionally train a 1B-parameter Qwen3.5-like model on 32B tokens targeting CPile from the same proxy model swarm and report its results in the last column.

\begin{table}[t]
\caption{
\textbf{Performance comparison of data mixing methods across $\mathbf{23\times}$ 1B par. OLMo2 target~models pretrained on 32B tokens.}
Normalized accuracy and bits-per-byte (BPB) are measured on a held-out set per target~benchmark.
Color denotes relative difference to the best baseline.
}
\label{tab:performance1b}
\vspace{0.5em}
\centering
\small
\begin{tblr}{
    colspec = {l *{12}{c}},
    colsep = 3.5pt,
    row{7} = {bg=black!10}
}
\toprule
 & \SetCell[c=5]{c} {\small\textbf{24 domains} ($512\times$ 30M models)} &  &  &  &  & \SetCell[c=7]{c} {\small\textbf{288 domains} ($1024\times$ 60M proxy models)} &  &  &  &  &  &  \\
\cmidrule[dashed, lr]{2-6} \cmidrule[dashed, lr]{7-13}
 & \SetCell[c=2]{c} ARC &  & \SetCell[c=2]{c} HellaSwag &  & CPile & \SetCell[c=2]{c} ARC &  & \SetCell[c=2]{c} HellaSwag &  & \SetCell[c=2]{c} HS + PIQA &  & CPile \\
\cmidrule[dashed, lr]{2-3} \cmidrule[dashed, lr]{4-5} \cmidrule[dashed, lr]{6-6} \cmidrule[dashed, lr]{7-8} \cmidrule[dashed, lr]{9-10} \cmidrule[dashed, lr]{11-12} \cmidrule[dashed, lr]{13-13}
\textbf{Method} & Acc$\uparrow$ & BPB$\downarrow$ & Acc$\uparrow$ & BPB$\downarrow$ & BPB$\downarrow$ & Acc$\uparrow$ & BPB$\downarrow$ & Acc$\uparrow$ & BPB$\downarrow$ & Acc$\uparrow$ & BPB$\downarrow$ & BPB$\downarrow$ \\
\midrule
Prior & .437 & .864 & .543 & .788 & .888 & .436 & .870 & .541 & .788 & .634 & .902 & .890 \\
\regmix & .496 & \textbf{.780} & .563 & \textbf{.781} & \textbf{.842} & .449 & .850 & .547 & .784 & .638 & .896 & .878 \\
\olmix & .486 & .785 & .550 & .782 & .860 & .435 & .872 & .550 & .783 & .641 & .892 & .850 \\
\method & \textbf{.510} & .785 & \textbf{.568} & .785 & .847 & \textbf{.511} & \textbf{.757} & \textbf{.615} & \textbf{.758} & \textbf{.690} & \textbf{.855} & \textbf{.834} \\
 & \impg{+2.8\%} & \impr{$-$0.6\%} & \impg{+0.9\%} & \impr{$-$0.5\%} & \impr{$-$0.6\%} & \impg{+13.8\%} & \impg{+10.9\%} & \impg{+11.8\%} & \impg{+3.2\%} & \impg{+7.6\%} & \impg{+4.1\%} & \impg{+1.9\%} \\
\bottomrule
\end{tblr}
\end{table}

\begin{table}[t]
\caption{
\textbf{Performance of data mixing methods across $\mathbf{10\times}$ 4B Qwen3.5 target models pretrained on 80B tokens.}
Mixes are optimized with 40 domains for each target benchmark and evaluated on its held-out set and a held-out benchmark.
The additional CPile column reports $4\times$ 1B Qwen3.5 models optimized for held-out Common Pile validation loss.
Color denotes relative diff. to the best baseline.
}
\label{tab:performance4b}
\vspace{0.5em}
\centering
\small
\begin{tblr}{
    colspec = {l *{13}{c}},
    colsep = 2.6pt,
    row{7} = {bg=black!10},
}
\toprule
 & \SetCell[c=4]{c} \textbf{Target: HellaSwag} &  &  &  & \SetCell[c=4]{c} \textbf{Target: MBPP} (code) &  &  &  & \SetCell[c=4]{c} \textbf{Target: GSM8K} (math) &  &  &  & \textbf{CPile} \\
\cmidrule[dashed, lr]{2-5} \cmidrule[dashed, lr]{6-9} \cmidrule[dashed, lr]{10-13} \cmidrule[dashed, lr]{14-14}
 & \SetCell[c=2]{c} HellaSwag &  & \SetCell[c=2]{c} PIQA &  & \SetCell[c=2]{c} MBPP &  & \SetCell[c=2]{c} HumanEval &  & \SetCell[c=2]{c} GSM8K &  & \SetCell[c=2]{c} MATH &  & CPile \\
\cmidrule[dashed, lr]{2-3} \cmidrule[dashed, lr]{4-5} \cmidrule[dashed, lr]{6-7} \cmidrule[dashed, lr]{8-9} \cmidrule[dashed, lr]{10-11} \cmidrule[dashed, lr]{12-13} \cmidrule[dashed, lr]{14-14}
\textbf{Method} & {\small Acc$\uparrow$} & {\small BPB$\downarrow$} & {\small Acc$\uparrow$} & {\small BPB$\downarrow$} & {\small P@1$\uparrow$} & {\small BPB$\downarrow$} & {\small P@1$\uparrow$} & {\small BPB$\downarrow$} & {\small EM$\uparrow$} & {\small BPB$\downarrow$} & {\small EM$\uparrow$} & {\small BPB$\downarrow$} & {\small BPB$\downarrow$} \\
\midrule
Prior & .632 & .743 & .764 & .952 & .304 & .375 & .213 & .285 & .214 & .403 & .162 & .463 & .814 \\
\regmix & .659 & .735 & .775 & .933 & .304 & .379 & .213 & .290 & .135 & .418 & .124 & .487 & .808 \\
\olmix & .667 & .731 & \textbf{.783} & .932 & .302 & .368 & .244 & .277 & .294 & .391 & .202 & .446 & .801 \\
\method & \textbf{.696} & \textbf{.716} & .782 & \textbf{.920} & \textbf{.444} & \textbf{.334} & \textbf{.280} & \textbf{.258} & \textbf{.337} & \textbf{.383} & \textbf{.266} & \textbf{.427} & \textbf{.797} \\
 & \impg{+4.4\%} & \impg{+2.0\%} & \impr{$-$0.1\%} & \impg{+1.3\%} & \impg{+46.1\%} & \impg{+9.3\%} & \impg{+15.0\%} & \impg{+7.0\%} & \impg{+14.4\%} & \impg{+2.1\%} & \impg{+31.6\%} & \impg{+4.3\%} & \impg{+0.5\%} \\
\bottomrule
\end{tblr}
\end{table}

\section{Theoretical Analysis}\label{sec:theoretical-analysis}

We formalize the domain attribution view of data mixing from first principles in Appendix~\ref{app:theoretical-analysis} and summarize these results here.
Starting from linear datamodels~\citep{ilyas2022datamodels}, we show that the expected outcome of a proxy run is linear in its domain mix, with coefficients that average sample-level attribution scores (Proposition~\ref{prop:domain_attribution}, Corollary~\ref{cor:singleton}).
Domain attribution thus interpolates between data attribution and a single-domain corpus, preserving the linear structure in between.
We further relate it to the Shapley and Banzhaf data values of cooperative game theory in Remark~\ref{rem:banzhaf}.

With a finite budget of proxy runs, estimating the domain attribution is a linear regression whose error is governed by the design $\boldsymbol{\Pi}$ (Lemma~\ref{lem:ols_error}).
Dirichlet sampling inflates this error as the concentration $\gamma$ grows and keeps its mixes close to the prior (Theorem~\ref{thm:dirichlet_design}).
With many domains, Dirichlet mixes cannot stay far from the prior because sparse mixes at a small $\gamma$ violate the per-domain caps and are rejected.
We therefore tune $\gamma$ empirically to a near-prior range (Section~\ref{sec:experiments-tuning}), for which the expected error of Dirichlet sampling is theoretically several times the lowest error possible with equal caps~(Appendix~\ref{app:dirichlet-theory}).
This matches Figure~\ref{fig:plot_budget_main}, where both samplers lead to reliable fits at 24 domains, but no surrogate fits the Dirichlet design well at 288 domains.
In theory, vertex sampling without caps trains each proxy model on a single domain and thus estimates every domain attribution independently and equally well, which no other design can improve on (Proposition~\ref{prop:vertex_design}).
In practice, the caps force every mix to spread over several domains and reveal less about each of them, but vertex sampling remains the best possible design in expectation for equal caps~(Theorem~\ref{thm:capped_vertex}).

Under these idealizations, Theorems~\ref{thm:dirichlet_design} and~\ref{thm:capped_vertex} help explain why the linear fit of \methodx can be more reliable than one based on Dirichlet sampling.
Dirichlet mixes concentrate around the prior, while vertices carry as much evenly spread information as the caps allow.

\section{Conclusion}

We demonstrate that an interpretable linear model suffices for data mixing under vertex sampling and validate it across language model pretraining benchmarks.
We further develop the theoretical foundations of \methodx and relate data mixing to the attribution literature.
While this paper contributes to the interpretability of data mixing, our results suggest that optimizing a vertex-sampled matrix with a nonlinear model could yield even better target model performance, which we leave for future work.
We discuss limitations of our work in Appendix~\ref{app:limitations}, most importantly the cost of training a proxy swarm and the single training run per target model.

\begin{ack}
Hubert Baniecki was supported by the Foundation for Polish Science
(FNP), and the state budget within the Polish Ministry of Education and Science program ``Pearls of Science'' project number PN/01/0087/2022.
Work on this project is financially supported by the Foundation for Polish Science (FNP) grant `Centre for Credible AI' No. FENG.02.01-IP.05-0058/24, and we gratefully acknowledge the Polish high-performance computing infrastructure PLGrid (HPC Center: ACK Cyfronet AGH) for providing computer facilities and support within the computational grant no. PLG/2026/019857.
\end{ack}

\bibliography{references}
\bibliographystyle{plainnat}

\clearpage
\appendix
\section*{Appendix ``LinMix: Data Mixing as Linear Domain Attribution''}

\startcontents[sections]
\noindent\rule{\textwidth}{1pt}
\printcontents[sections]{l}{1}{\setcounter{tocdepth}{2}}
\noindent\rule{\textwidth}{1pt}

\clearpage
\section{Related Work}\label{app:related-work}

\paragraph{Design of experiments.}
Sampling $\boldsymbol{\Pi}$ is a classic problem in statistics.
Because each row is a vector of domain proportions that sum to one, the proxy swarm is a mixture experiment~\citep{scheffe1958experiments}.
Choosing inputs that minimize the variance of a regression estimate is the subject of optimal experimental design.
A design is \emph{D-optimal} when it maximizes $\det(\boldsymbol{\Pi}^\top \boldsymbol{\Pi})$, which minimizes the volume of the confidence ellipsoid of $\hat{\boldsymbol{\beta}}$~\citep{kiefer1959regression}.
For approximate designs, \citet{kiefer1960equivalence} show that maximizing this determinant also minimizes the worst-case prediction variance.
The D-optimal design for the first-degree mixture model places equal weight on the vertices of the simplex~\citep{kiefer1961optimum}.
When upper bounds such as our caps $u_k$ truncate the simplex, the extreme vertices designs of \citet{mclean1966extreme} likewise use the vertices of the constrained region.
\methodx follows this principle at a scale where computing exact optimal designs on $\mathcal{P}_{\text{cap}}$ is impractical, sampling vertices generated by uniformly random domain permutations.
We prove that this sampling scheme is optimal in expectation on $\mathcal{P}_{\text{cap}}$ with equal caps (Theorem~\ref{thm:capped_vertex}) and leave the case of unequal caps to future work.
In Section~\ref{sec:experiments-tuning}, we tune Dirichlet sampling by minimizing the condition number of $\boldsymbol{\Pi}$, a criterion closely related to D-optimality and $\det(\boldsymbol{\Pi}^\top \boldsymbol{\Pi})$.
Both degrade together because nearly-collinear mixes reduce the smallest singular value of $\boldsymbol{\Pi}$, thereby collapsing the determinant and inflating the condition number.
\methodx also samples proxy mixes in pairs by filling domains in a random order and then in reverse order, following the idea of antithetic variates~\citep{hammersley1956new}.
This relates to pairing sampled coalitions with their complements that has improved Shapley value approximation~\citep{fumagalli2026odd}, though our capped mixes need not be exact complements.
More broadly, regression design is of interest for improving attribution, e.g. \citet{musco2025provably} improve regression-based Shapley value approximation via leverage-score sampling of coalitions.

\paragraph{Data mixing.}
Data mixing methods differ in how they estimate domain value and spend compute to choose a mix~\citep{liu2026rethinking}.
Online methods adapt the mixture during training.
DoReMi~\citep{xie2023doremi} learns weights during proxy training through group distributionally robust optimization, DoGE~\citep{fan2024doge} and \citet{tan2026fast} use gradient-based updates guided by generalization or validation performance, and PiKE~\citep{li2025pike} re-weights data sources from their gradients.
MixMin~\citep{thudi2025mixmin} optimizes a weighted ensemble of proxy predictions, while Chameleon~\citep{xie2025chameleon} derives weights from relationships between domain embeddings.
Data Mixing Laws~\citep{ye2025data}, AutoScale~\citep{kang2025autoscale}, and \citet{shukor2025scaling} use scaling relationships to predict mixture performance or select mixtures at larger training scales.
Mixture optimization also extends to tokenizers~\citep{won2026trex} and multimodal midtraining~\citep{wen2026mixatlas}.
Adaptive search allocates new proxy runs to promising mixtures through iterative predictor fitting~\citep{diao2025nemotron} or Bayesian optimization~\citep{ouyang2025admirebayesopt,yen2025data}.
\methodx instead fixes its proxy mixtures before observing their outcomes, using vertex sampling to estimate domain effects across the feasible region.
This design enables a useful linear fit across hundreds of domains; combining it with adaptive search is a direction for future work.

\paragraph{Theory of optimal regression-based data mixing.}
Regression-based data mixing methods select the target mix by minimizing a fitted surrogate model, offering little theoretical insight into the optimal mix itself~\citep{liu2025regmix,wettig2025weborganizer,chen2026olmix}.
Recently, \citet{chen2026olmix} bound the performance loss when a previously optimized mix is reused after the domain set changes, taking the mixing law and its exact minimization as given.
\citet{thrush2025improving} rank domains by correlating the per-domain losses of public models with their benchmark scores instead, and prove that such rankings are optimally used by saturating the caps in rank order.
In our framework, this rule is the vertex solution of Eq.~\eqref{eq:objective_r} at $r = 0$, with the ranking estimated from a designed experiment on training mixtures rather than a correlation across uncontrolled models.
For $r > 0$, the objective is a weighted \emph{isoelastic} utility, also known as the alpha-fair family of resource allocation~\citep{mo2000fair}, whose marginal utility grows unbounded at zero weight, so the maximizer places positive weight on every positively attributed domain (Proposition~\ref{prop:full-support}).
Corollary~\ref{cor:optimal_mix} solves the KKT conditions in closed form, yielding a capped water-filling allocation~\citep[cf.][]{boyd2004convex}.

\paragraph{Data attribution.}
Our work relates to the parallel literature on data attribution, which typically studies the influence of individual training examples rather than entire domains~\citep{ghorbani2019data,feldman2020neural,wang2023data,li2023robust}.
We identify a connection between regression-based data mixing and datamodels~\citep{ilyas2022datamodels,saunshi2023understanding}, which predict model outputs from training-set membership and show that a linear model can capture these dependencies.
DsDm~\citep{engstrom2024dsdm} turns this predictive view into data selection, using gradient-based TRAK estimates~\citep{park2023trak} to rank individual examples and select a fixed-size subset.
\methodx shares this principle but estimates domain proportions from proxy runs and allocates tokens under domain caps.
Using these estimates for a larger target model requires that data effects transfer across scales, as studied by \citet{khaddaj2025smalltolarge}.
Transfer depends on proxy training hyperparameters, as lower learning rates improve agreement with the dataset rankings of larger models~\citep{wang2026can}.
With a limited number of runs~\citep{wang2025data}, the question becomes which mixtures reveal enough about each domain to support the allocation, which vertex sampling answers by design.

\paragraph{Data selection and filtering.}
Data quality and filtering methods improve training corpora by selecting or weighting individual examples.
DataComp~\citep{gadre2023datacomp,li2024datacomp} benchmarks how data curation affects large model performance.
Methods that combine selection criteria include Meta-rater~\citep{zhuang2025metarater}, which learns combinations of quality scores from proxy runs, and \citet{fan2026joint}, who jointly optimize quality and diversity through learned selection masks, while Group-MATES~\citep{yu2025group} further accounts for interactions between examples.
During training, selection can adapt by coordinating multiple criteria~\citep{bai2025efficient}, updating example weights~\citep{yu2025llm}, or scoring examples by their optimizer updates~\citep{wang2026opus}.
For fine-tuning, \citet{ling2025diversity} select data using a diversity reward, while LimaCost~\citep{moon2025limacost} measures instruction-data value against a reference dataset using gradients.
Filtering and mixing can be combined by curating the corpus before optimizing its domain mix, as in our experiments.

\section{Proofs}\label{app:proofs}

\begin{greybox}
\isosupport*
\end{greybox}

\begin{proof}
Throughout, let $\tilde{\boldsymbol{\pi}}$ denote the feasible mix that spreads the budget over $\mathcal{A}$ in proportion to the caps, $\tilde{\pi}_k \coloneqq u_k / u_{\mathcal{A}}$ for $k \in \mathcal{A}$ and $\tilde{\pi}_k \coloneqq 0$ otherwise.
Its entries sum to one and $u_{\mathcal{A}} > 1$ gives $\tilde{\pi}_k < u_k$, so it is feasible and puts positive weight on every domain in $\mathcal{A}$.

\emph{Existence.}
For $r \in (0,1)$, the objective $\nu_r(\boldsymbol{\pi}) \coloneqq \sum_{k=1}^d \hat{\beta}_k \frac{\pi_k^{1-r}-1}{1-r}$ is continuous on the compact set $\mathcal{P}_{\text{cap}}$, so a maximizer exists.
For $r = 1$, the proposition defines the objective as $\nu_1(\boldsymbol{\pi}) \coloneqq \sum_{k \in \mathcal{A}} \hat{\beta}_k \ln \pi_k$ with the convention $\ln 0 = -\infty$.
This convention continues the case $r < 1$, where all maximizers lie on the face $\{\pi_k = 0 \text{ for } k \notin \mathcal{A}\}$.
On this face the omitted terms are constants that do not affect the maximizer, and the remaining sum converges pointwise to $\nu_1$ as $r \to 1$.
$\nu_1$ is upper semicontinuous and bounded above by zero, so it attains a finite maximum, as $\nu_1(\tilde{\boldsymbol{\pi}}) > -\infty$.

\emph{Positive weight on every domain in $\mathcal{A}$.}
The utility of a domain in $\mathcal{A}$ has an unbounded derivative at zero weight, so leaving such a domain empty cannot be optimal.
To make this precise, let $\boldsymbol{\pi}^\star$ be any maximizer and suppose $\pi^\star_j = 0$ for some $j \in \mathcal{A}$.
For $r = 1$ this gives $\nu_1(\boldsymbol{\pi}^\star) = -\infty < \nu_1(\tilde{\boldsymbol{\pi}})$, a contradiction.
For $r \in (0,1)$, pick any $i$ with $\pi^\star_i > 0$ and move a small weight $\varepsilon \in (0, \min(\pi^\star_i/2,\, u_j))$ from $i$ to $j$.
Only two terms change.
The $j$-th gains $\hat{\beta}_j \frac{\varepsilon^{1-r}}{1-r} > 0$.
The $i$-th loses at most $\hat{\beta}_i^+ (\pi^\star_i/2)^{-r}\, \varepsilon$ with $\hat{\beta}_i^+ \coloneqq \max(\hat{\beta}_i, 0)$, by the mean value theorem on $[\pi^\star_i/2,\, \pi^\star_i]$.
The gain is roughly $\varepsilon^{1-r}$ while the loss is only proportional to $\varepsilon$, so the move strictly improves the objective for small $\varepsilon$, again a contradiction.
Hence $\pi^\star_k > 0$ for all $k \in \mathcal{A}$.

\emph{Zero weight outside $\mathcal{A}$.}
Conversely, weight outside $\mathcal{A}$ adds no value and is better spent inside.
Suppose $\pi^\star_j > 0$ for some $j \notin \mathcal{A}$.
Because $\sum_{k \in \mathcal{A}} \pi^\star_k \le 1 < u_{\mathcal{A}}$, some domain $i \in \mathcal{A}$ sits strictly below its cap.
Moving the weight $\varepsilon \coloneqq \min(\pi^\star_j,\, u_i - \pi^\star_i) > 0$ from $j$ to $i$ is feasible.
The $j$-th term does not decrease, since $\hat{\beta}_j \le 0$ makes it non-increasing in $\pi_j$, and for $r = 1$ it is absent from $\nu_1$ altogether.
The $i$-th term strictly increases, a contradiction.
Hence $\pi^\star_k = 0$ for all $k \notin \mathcal{A}$.

\emph{Uniqueness.}
By the two steps above, every maximizer lies in the convex set $\mathcal{F} \coloneqq \{\boldsymbol{\pi} \in \mathcal{P}_{\text{cap}} : \pi_k = 0 \text{ for } k \notin \mathcal{A}\}$ and is strictly positive on $\mathcal{A}$.
On $\mathcal{F}$, the objective is a sum of strictly concave functions of the coordinates in $\mathcal{A}$, since $\hat{\beta}_k > 0$ and $r \in (0,1]$.
Two distinct maximizers would therefore be beaten by their midpoint, so the maximizer is unique.

\emph{Linear limit.}
At $r = 0$, the objective $\nu_0(\boldsymbol{\pi}) = \hat{\boldsymbol{\beta}}^\top \boldsymbol{\pi} - \hat{\boldsymbol{\beta}}^\top \boldsymbol{1}$ is linear, so it is maximized at a vertex of the truncated simplex.
The greedy fill solving Eq.~\eqref{eq:linear_objective} can then exhaust the budget before reaching the lowest-value domains in $\mathcal{A}$ and leave them at zero weight.
For example, $\hat{\boldsymbol{\beta}} = (2,1,-3)$ with $u = (1,1,1)$ yields the optimum $(1,0,0)$.
\end{proof}

\begin{greybox}
\coroptimalmix*
\end{greybox}

\begin{proof}
Fix $r \in (0,1]$.
The claimed strict concavity on the support was established in the uniqueness step of Proposition~\ref{prop:full-support}, as each summand over $\mathcal{A}$ is strictly concave.
By Proposition~\ref{prop:full-support}, the maximizer $\boldsymbol{\pi}^\star$ is unique, lies in $\mathcal{F}$, and satisfies $\pi^\star_k > 0$ for all $k \in \mathcal{A}$.
On $\mathcal{F}$, the problem maximizes this concave objective under the affine constraints $\sum_{k \in \mathcal{A}} \pi_k = 1$ and $\pi_k \le u_k$.
Because every $\pi^\star_k$ is positive, the objective is differentiable at $\boldsymbol{\pi}^\star$ and the constraints $\pi_k \ge 0$ are inactive, so $\boldsymbol{\pi}^\star$ satisfies the KKT conditions~\citep[cf.][]{boyd2004convex} with a scalar $\omega$ for the budget and multipliers $\mu_k \ge 0$ for the caps.
The marginal utility of domain $k$ is $\hat{\beta}_k \pi_k^{-r}$, so stationarity reads $\hat{\beta}_k (\pi^\star_k)^{-r} = \omega + \mu_k$ for every $k \in \mathcal{A}$.
In words, the optimum equalizes marginal utilities at the common level $\omega$, except where a cap binds.
Because $u_{\mathcal{A}} > 1 = \sum_{k \in \mathcal{A}} \pi^\star_k$, some domain lies strictly below its cap.
For this domain, complementary slackness gives $\mu_k = 0$, so $\omega = \hat{\beta}_k (\pi^\star_k)^{-r} > 0$.
Now the closed form can be read off.
If $\pi^\star_k < u_k$, then $\mu_k = 0$ and stationarity gives $\pi^\star_k = (\hat{\beta}_k/\omega)^{1/r}$.
If $\pi^\star_k = u_k$, then $\mu_k \ge 0$ gives $(\hat{\beta}_k/\omega)^{1/r} \ge u_k$.
In both cases, $\pi^\star_k = \min\big(u_k,\, (\hat{\beta}_k/\omega)^{1/r}\big)$.
Since $\boldsymbol{\pi}^\star$ is feasible, summing this form over $k \in \mathcal{A}$ shows that $\omega$ enforces the normalization $\sum_{k \in \mathcal{A}} \min\big(u_k,\, (\hat{\beta}_k/\omega)^{1/r}\big) = 1$.
Finally, $\omega$ is unique.
The sum $\sum_{k \in \mathcal{A}} \min\big(u_k,\, (\hat{\beta}_k/\omega)^{1/r}\big)$ is continuous and non-increasing in $\omega$.
At any $\omega$ where it equals $1 < u_{\mathcal{A}}$, some term lies strictly below its cap and there decreases strictly in $\omega$.
Therefore the value $1$ is attained at exactly one $\omega$.
\end{proof}

\section{Experimental Setup}
\label{app:experimental-setup}

\paragraph{Pretraining data.}
We tokenize data into single-domain sequences of 2048 tokens with the GPT-NeoX-20B tokenizer~\citep{black2022gptneox}, which has a vocabulary of 50{,}277 tokens.
This way, the WebOrganizer corpus~\citep{wettig2025weborganizer} yields 195B tokens, and the 40-domain corpus adds 168B tokens of Stack-Edu~\citep{allal2025smollm} across 15 programming languages, plus the 53B-token highest-quality subset of Nemotron-CC-Math~\citep{mahabadi2026nemotron}, bringing the total data size to 416B. 
Figure~\ref{fig:prior_mix} shows the resulting prior data mixes.

\begin{figure}[ht]
    \centering
    \includegraphics[width=\linewidth]{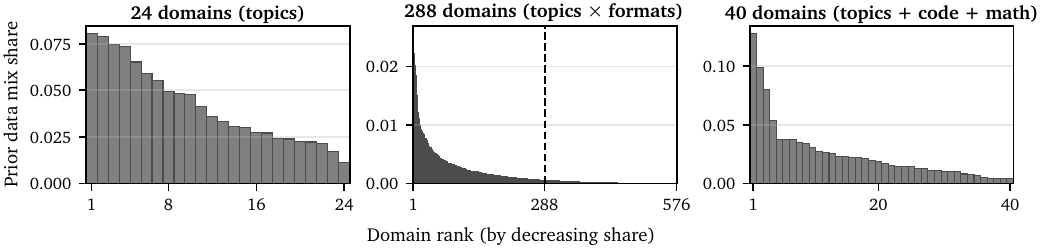}
    \caption{
    \textbf{Prior data mix distributions.}
    Domain shares for the three setups: 24 domains (topics), 288 domains formed from topic and format pairs, and 40 domains extending topics with code and math data.
    The dashed line marks the removal of the distribution's tail, where the kept larger half of all 576 pairs covers 94\% of the corpus.
    }
    \label{fig:prior_mix}
\end{figure}

\paragraph{Sampling data mixes.} 
The caps $\kappa$ become $5.98$ and $5.63$ for the 32B target token budget at 24 and 288 domains, and $\kappa=3.40$ at 40 domains.
We initially derived the latter for a 120B budget, allowing us to further train the models beyond the 80B tokens used.
Dirichlet sampling uses $\tau = 2$ and $[\gamma_{\mathrm{lo}}, \gamma_{\mathrm{hi}}] = [0.1, 10]$, $[40, 4000]$, and $[100, 10^{4}]$ at 24, 288, and 40 domains, respectively.
The 40-domain range follows the criterion of Section~\ref{sec:experiments-tuning}: the default parameters reject over 99.9\% of draws, and the condition number of $\boldsymbol{\Pi}$ is minimized near $\gamma = 10^{2}$.
We keep the $100\times$ span of the tuned range, over which the rejection rate falls from 0.96 to 0.01.
For both samplers, we draw $10^{4}$ designs and keep the best-conditioned $\boldsymbol{\Pi}$, which gives the most stable regression coefficients.

\paragraph{Language models.}
Table~\ref{tab:hparams} lists model architectures and hyperparameters. 
Both model families follow their reference implementations~\citep{walsh2025olmo2,qwen2026qwen35}; Qwen3.5 interleaves three gated DeltaNet layers per attention layer.
All models use AdamW with $\beta_1 = 0.9$ and weight decay $0.1$.
The peak learning rate scales with the inverse cube root of the model parameter count~\citep{porian2024resolving}.
The schedule warms up linearly over the first 10\% of steps, decays by cosine to 10\% of the peak, and then linearly to zero over the final 10\%.
\begin{table}[ht]
\centering\small
\caption{
\textbf{Model architectures and training hyperparameters.} 
Parameters include embeddings. 
Qwen3.5 layer counts list gated DeltaNet + full-attention layers. 
Batch gives sequences per step.
}
\label{tab:hparams}
\vspace{0.5em}
\begin{tabular}{lrrrrrrrr}
\toprule
Model & Params & Layers & Heads & Hidden size & Batch & Tokens & Peak LR & $\beta_2$ \\
\midrule
OLMo2 30M   & 29.2M  & 4    & 8  & 256  & 64   & 1B   & $3.6\times10^{-3}$ & 0.99 \\
OLMo2 60M   & 57.6M  & 8    & 8  & 384  & 128  & 2B   & $2.9\times10^{-3}$ & 0.98 \\
OLMo2 100M  & 102.0M & 12   & 8  & 512  & 128  & 3.2B & $2.4\times10^{-3}$ & 0.97 \\
OLMo2 1B    & 1.41B  & 18   & 16 & 2048 & 512  & 32B  & $1.1\times10^{-3}$ & 0.95 \\
Qwen3.5 100M & 104.5M & 9+3  & 4  & 576  & 128  & 3.2B & $2.4\times10^{-3}$ & 0.97 \\
Qwen3.5 1B   & 1.33B  & 12+4 & 16 & 2048 & 512  & 32B  & $1.1\times10^{-3}$ & 0.95 \\
Qwen3.5 4B   & 3.70B  & 24+8 & 16 & 2560 & 1024 & 80B  & $7.1\times10^{-4}$ & 0.95 \\
\bottomrule
\end{tabular}
\end{table}

\paragraph{Evaluation.}
We use the LM Evaluation Harness~\citep{gao2024harness} for a standardized evaluation.
Proxy models are scored by BPB on the training split of each benchmark, and target models by accuracy and BPB on their held-out set (Table~\ref{tab:benchmarks}).
Shots are drawn from the training split, except MBPP and MATH, which use their standard fixed prompts, and HumanEval BPB, which relies on three hand-written examples.
For the two-benchmark targets, ARC (easy and challenge) and HS+PIQA, we fit a surrogate per benchmark, minimize the mean of their predictions, and report the mean of the two target model scores.
For the 4B models, math and code benchmarks are scored using greedy decoding, with exact match for math and pass@1 for code.
We draw two disjoint held-out sets from the Common Pile dataset~\citep{kandpal2025common}: a 150k-sequence split on which we score the proxy models to fit the surrogates, and a 300k-sequence split reserved for scoring the target models.
\begin{table}[h]
\centering\small
\caption{
\textbf{Benchmark splits, sample sizes, and few-shot counts.}
Dashes mark benchmarks used only as held-out checks.
}
\label{tab:benchmarks}
\vspace{0.5em}
\begin{tabular}{lllcl}
\toprule
Benchmark & Proxy split (size) & Held-out set (size) & Shots & Target metric \\
\midrule
ARC-Easy      & train (2,251)  & test (2,376)         & 5 & normalized acc. \\
ARC-Challenge & train (1,119)  & test (1,172)         & 5 & normalized acc. \\
HellaSwag     & train (39,905) & validation (10,042)  & 5 & normalized acc. \\
PIQA          & train (16,113) & validation (1,838)   & 5 & normalized acc. \\
MBPP          & train (374)    & test (500)           & 3 & pass@1 \\
GSM8K         & train (7,473)  & test (1,319)         & 5 & exact match \\
HumanEval     & ---            & test (164)           & 0 & pass@1 \\
MATH          & ---            & test (5,000)         & 4 & exact match \\
\bottomrule
\end{tabular}
\end{table}

\clearpage
\section{Additional Results}
\label{app:additional-results}

Figure~\ref{fig:plot_budget_extended} extends Figure~\ref{fig:plot_budget_main} with $R^2$ measurements, and Figure~\ref{fig:plot_budget_40domains} repeats it for the four targets of the 4B experiments at 40 domains, where vertex sampling reaches $\rho \geq 0.94$ with every surrogate at $m = 512$, while the tuned Dirichlet swarm stays at $\rho \leq 0.68$ on the three benchmarks and $0.80$ on CPile.
Figure~\ref{fig:scale_transfer} shows that the linear fit of the vertex swarm transfers across proxy scales.
Figures~\ref{fig:heatmap1b} and~\ref{fig:heatmap4b_cpile} visualize the interpretable domain weights of the optimized mixes at 24 and 288 domains and for the Common Pile loss at 40 domains.

Tables~\ref{tab:ablation_scaling} and~\ref{tab:ablation_r} ablate the two optimization choices of \methodx on the HellaSwag target.
Table~\ref{tab:ablation_scaling} tests our hypothesis that dividing the coefficients $\hat{\boldsymbol{\beta}}$ by their standard errors, which discounts uncertain domains, improves the mix.
The gain is at most half a point of accuracy for both the vertex swarm and the tuned Dirichlet swarm with $\gamma \in [40, 4000]$.
Table~\ref{tab:ablation_r} shows that the interpretable transformation of $\pi$ matters at 24 domains, where $\pi$ ($r = 0$) concentrates the mix on four domains at their caps $u_k$.
This support is too narrow for the target budget of 32B tokens, which a larger budget would spread through lower caps, and $\sqrt{\pi}$ or $\ln \pi$ ($r > 0$) spreads at the current one.

\begin{figure}[ht]
    \centering
    \includegraphics[width=\linewidth]{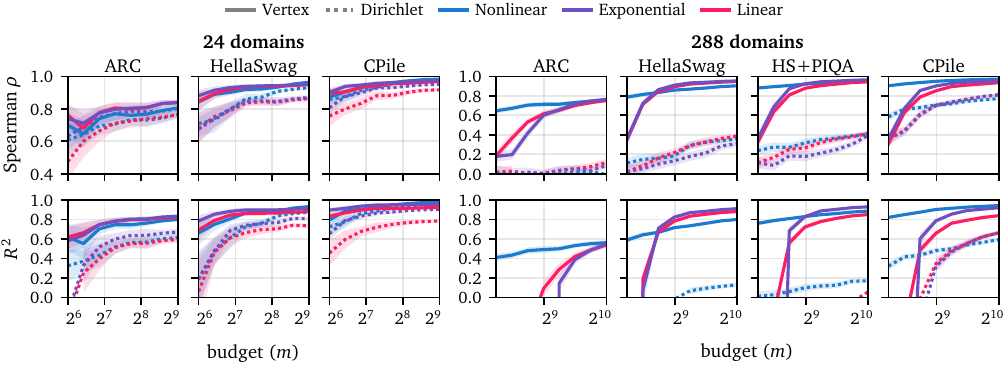}
    \caption{
    \textbf{Performance of optimizers under Vertex and Dirichlet samplers.}
    Extended Figure~\ref{fig:plot_budget_main}.
    Average Spearman correlation and $R^2$ across 5 random repetitions of 10-fold cross-validation under a fixed proxy model budget.
    Nonlinear is implemented using LightGBM, as in \regmix, and exponential fits the exponential law to the target, as in \olmix.
    }
    \label{fig:plot_budget_extended}
\end{figure}

\begin{figure}[ht]
    \centering
    \includegraphics[width=0.582\linewidth]{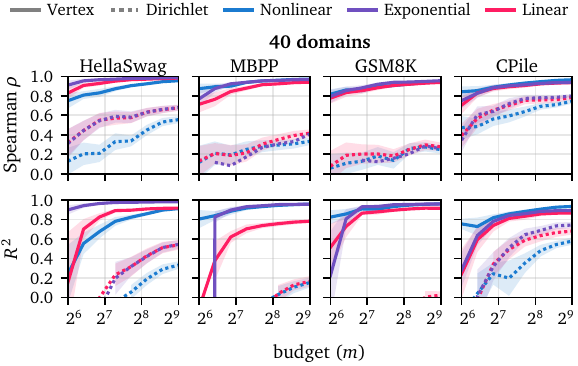}
    \caption{
    \textbf{Optimizer performance on the targets of the 4B experiments.}
    Figure~\ref{fig:plot_budget_extended} for the 40-domain setup, where $512\times$ 100M Qwen3.5 proxy models per sampler are fit to HellaSwag, MBPP, GSM8K, and the held-out Common Pile loss (CPile).
    Spearman correlation and $R^2$ are averaged across 5 random repetitions of 10-fold cross-validation.
    }
    \label{fig:plot_budget_40domains}
\end{figure}

\begin{figure}[ht]
    \centering
    \includegraphics[width=0.67\textwidth]{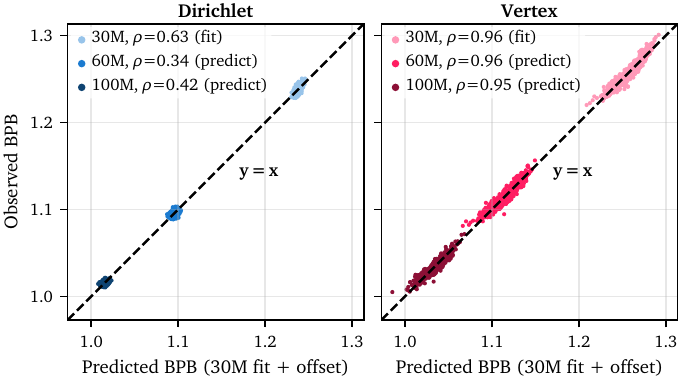}
    \caption{
    \textbf{The linear attribution fit transfers across proxy scales only under vertex sampling.}
    A linear model is fit to predict HellaSwag BPB from the mix over the $1024\times$ 30M proxy models at 288 domains, and predicts the outcomes of the 60M and 100M swarms trained on the same mixes.
    Absolute BPB drops with proxy scale, so we offset predictions by one constant per swarm without changing rankings.
    Dirichlet-sampled mixes yield a weak in-sample fit whose ranking does not transfer, whereas vertex sampling yields a ranking that transfers across model scales.
    }
    \label{fig:scale_transfer}
\end{figure}

\begin{figure}[ht]
    \centering
    \includegraphics[width=0.95\linewidth]{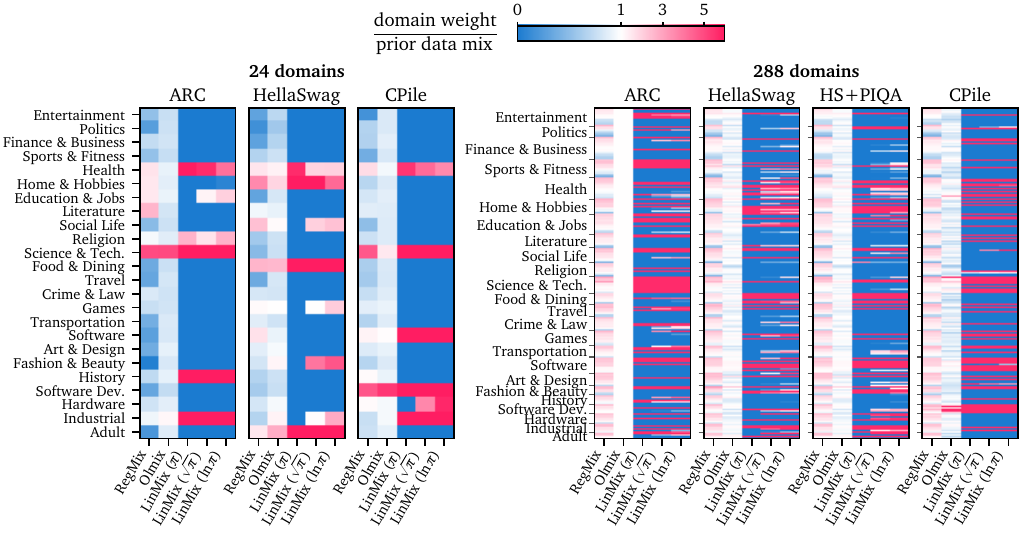}
    \caption{
    \textbf{Domain weights optimized for the 1B target benchmarks.}
    Domain weights relative to the prior mix across data mixing methods and targets on the 200B-token corpus split into 24 and 288 domains, including the held-out Common Pile loss (CPile).
    For 288 domains, rows retain the 24-topic order and sort each topic's formats by increasing share of the prior data mix.
    Topic names are at the group centers, and since a topic group with few formats is shorter than its label, it can overlap vertically.
    Beyond $\pi\;(r=0)$, the support of the target mix is the set of all positively attributed domains; the mixes at $\sqrt{\pi}\;(r=0.5)$ and $\ln\pi\;(r=1)$ thus share this support and differ only in the domain weight magnitudes~(Proposition~\ref{prop:full-support}).
    Based on this, Table~\ref{tab:performance1b} trains the \methodx $\ln\pi$ mix for HellaSwag at 24 domains, the $\sqrt{\pi}$ mix for CPile at 24 domains, and the $\pi$ mix otherwise.
    }
    \label{fig:heatmap1b}
\end{figure}

\begin{figure}[ht]
    \centering
    \includegraphics[width=0.412\linewidth]{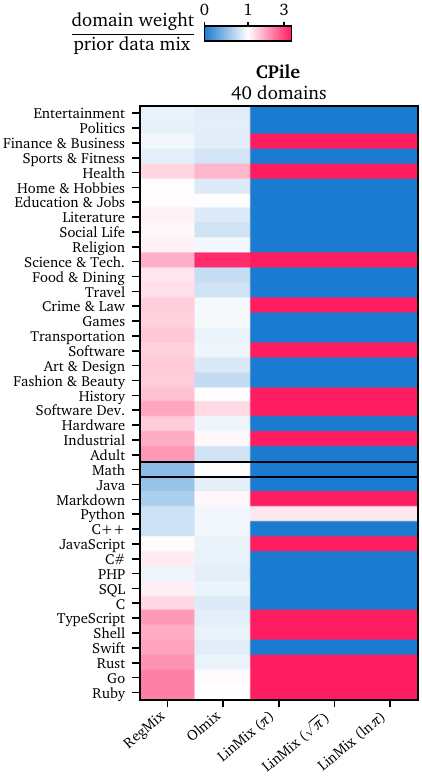}
    \caption{
    \textbf{Domain weights optimized for validation loss at 40 domains.}
    Domain weights relative to the prior mix across data mixing methods for BPB on the Common Pile.
    The three \methodx columns vary the utility transformation. 
    }
    \label{fig:heatmap4b_cpile}
\end{figure}

\begin{table}[ht]
\small
\begin{minipage}[t]{0.47\textwidth}
    \centering
    \caption{Ablation of linear model scaling and proxy sampling at 288 domains on HellaSwag.}
    \label{tab:ablation_scaling}
    \vspace{0.5em}
    \begin{tblr}{
        colspec = {l l c c},
        hline{4} = {dashed, 0.5pt},
    }
    \toprule
    \textbf{Sampling} & \textbf{Scaling} & Acc$\uparrow$ & BPB$\downarrow$ \\
    \midrule
    \SetCell[r=2]{l} Dirichlet & $\hat{\beta}$ & .597 & .764 \\
     & $\hat{\beta}/\mathrm{SE}$ & .602 & .761 \\
    \SetCell[r=2]{l} Vertex & $\hat{\beta}$ & .615 & .758 \\
     & $\hat{\beta}/\mathrm{SE}$ & .617 & .759 \\
    \bottomrule
    \end{tblr}
\end{minipage}\hfill
\begin{minipage}[t]{0.49\textwidth}
    \centering
    \caption{Ablation of the utility transformation via curvature $r$ in \methodx on the HellaSwag target.}
    \label{tab:ablation_r}
    \vspace{0.5em}
    \begin{tblr}{
        colspec = {r *{4}{c}},
    }
    \toprule
     & \SetCell[c=2]{c} \textbf{24 domains} & & \SetCell[c=2]{c} \textbf{288 domains} & \\
    \cmidrule[dashed, lr]{2-3} \cmidrule[dashed, lr]{4-5}
    \textbf{Utility} & Acc$\uparrow$ & BPB$\downarrow$ & Acc$\uparrow$ & BPB$\downarrow$ \\
    \midrule
    $\pi$ ($r=0$) & .543 & .795 & .615 & .758 \\
    $\sqrt{\pi}$ ($r=.5$) & .565 & .782 & .611 & .758 \\
    $\ln\pi$ ($r=1$) & .568 & .785 & .610 & .759 \\
    \bottomrule
    \end{tblr}
\end{minipage}
\end{table}

\clearpage

\section{Theoretical Analysis} \label{app:theoretical-analysis}

We formalize the domain attribution view of data mixing from first principles.
Starting from linear datamodels, we show that the expected outcome of a proxy run is linear in its domain mix, with coefficients that average sample-level attribution scores (Proposition~\ref{prop:domain_attribution}, Corollary~\ref{cor:singleton}), and that estimating them is a linear regression whose error is governed by the design $\boldsymbol{\Pi}$ (Lemma~\ref{lem:ols_error}).
Dirichlet sampling inflates this error as the concentration $\gamma$ grows and keeps its mixes close to the prior, and also as the number of domains grows when the prior is more skewed than the temperature $\tau$ can flatten (Theorem~\ref{thm:dirichlet_design}).
Vertex sampling is instead optimal on the simplex~(Proposition~\ref{prop:vertex_design}) and, in expectation, on the simplex with equal caps~(Theorem~\ref{thm:capped_vertex}).

\subsection{From datamodels to domain attribution}\label{app:domainmodels}

Datamodeling~\citep{ilyas2022datamodels} attributes model outcomes to training data counterfactually, in a regime where many training runs on random subsets of the data are feasible.
Fix a learning algorithm and a target input $\boldsymbol{z}$, e.g. a benchmark question, and let $y(S; \boldsymbol{z})$ denote the loss on $\boldsymbol{z}$ of a model trained on a subset $S$ of the corpus $\mathcal{D} \coloneqq \mathcal{D}_1 \cup \cdots \cup \mathcal{D}_d$ of $n$ sequences.
A datamodel is fit to the pairs $(\boldsymbol{1}_{S_i}, y(S_i; \boldsymbol{z}))$ collected over runs $i = 1, \ldots, m$, where $\boldsymbol{1}_{S_i} \in \{0, 1\}^n$ indicates the sequences included in run $i$.
Empirically, a linear fit predicts the outcomes of held-out runs well and motivates the model
\begin{align}
y(S; \boldsymbol{z}) = \langle \boldsymbol{\beta}(\boldsymbol{z}), \boldsymbol{1}_S \rangle + \eta(\boldsymbol{z}), && (\text{Datamodel})
\label{eq:linear_datamodels}
\end{align}
where $\boldsymbol{\beta}(\boldsymbol{z}) \in \mathbb{R}^n$ collects the attribution scores, i.e., the marginal effect of including each sequence on the loss at $\boldsymbol{z}$, and $\eta(\boldsymbol{z})$ is a noise term with zero mean given $S$ that captures the randomness of training.
Follow-up work relies on this local linearity to relate attribution scores to model outcomes~\citep{karchmer2025efficiently}, and we adopt it as the assumption underlying this analysis.
As all subsets considered below have the same size, an intercept can be absorbed into $\boldsymbol{\beta}(\boldsymbol{z})$ and is omitted.

Consider a proxy run on the mix $\boldsymbol{\pi}$ that trains on $b \coloneqq b_{\text{proxy}}$ sequences, of which $b \pi_k$ are drawn uniformly without replacement from domain $\mathcal{D}_k$, independently across domains, where we assume that $b \pi_k \le n_k$ is an integer.
Let its outcome $y$ be the benchmark loss, i.e., the average of $y(S; \boldsymbol{z})$ over the set of benchmark inputs $\mathcal{Z}$.
Let $\boldsymbol{\beta}^{\mathcal{Z}} \coloneqq \frac{1}{|\mathcal{Z}|} \sum_{\boldsymbol{z} \in \mathcal{Z}} \boldsymbol{\beta}(\boldsymbol{z})$ denote the benchmark-averaged attribution scores, and let $\mu_k$ and $\sigma_k^2 \coloneqq \frac{1}{n_k} \sum_{s \in \mathcal{D}_k} (\beta^{\mathcal{Z}}_s - \mu_k)^2$ denote their mean and variance over the sequences of domain $k$.
Proposition~\ref{prop:domain_attribution} formalizes the data attribution view of data mixing.

\begin{greybox}
\begin{proposition}[\textbf{Datamodels coarsen to domain attribution}]
\label{prop:domain_attribution}
Under the datamodel of Eq.~\eqref{eq:linear_datamodels}, the expected outcome of a proxy run is linear in its mix.
Specifically, with $\bar{\beta}_k \coloneqq b\, \mu_k$,
\begin{equation}
\mathbb{E}[y \mid \boldsymbol{\pi}] = \langle \bar{\boldsymbol{\beta}}, \boldsymbol{\pi} \rangle
\qquad \text{and} \qquad
\mathrm{Var}(y \mid \boldsymbol{\pi}) \le \sigma_\eta^2 + b \max_k \sigma_k^2,
\end{equation}
where $\sigma_\eta^2$ bounds the variance of the benchmark-averaged noise $\frac{1}{|\mathcal{Z}|} \sum_{\boldsymbol{z} \in \mathcal{Z}} \eta(\boldsymbol{z})$ given $S$.
\end{proposition}
\end{greybox}

\begin{proof}
\emph{Linearity.}
For a fixed subset $S$, averaging Eq.~\eqref{eq:linear_datamodels} over $\boldsymbol{z} \in \mathcal{Z}$ gives $y = \langle \boldsymbol{\beta}^{\mathcal{Z}}, \boldsymbol{1}_S \rangle + \bar{\eta}$ with $\bar{\eta} \coloneqq \frac{1}{|\mathcal{Z}|} \sum_{\boldsymbol{z} \in \mathcal{Z}} \eta(\boldsymbol{z})$ and $\mathbb{E}[\bar{\eta} \mid S] = 0$, because the same subset produces the loss at every $\boldsymbol{z}$ and the average commutes with the inner product.
Given $\boldsymbol{\pi}$, the subset $S$ contains $b \pi_k$ sequences drawn uniformly without replacement from $\mathcal{D}_k$, so every $s \in \mathcal{D}_k$ is included with probability $\Pr[s \in S] = b \pi_k / n_k$.
By linearity of expectation,
\begin{equation*}
\mathbb{E}[y \mid \boldsymbol{\pi}] = \sum_{k=1}^d \sum_{s \in \mathcal{D}_k} \beta^{\mathcal{Z}}_s \, \frac{b \pi_k}{n_k} = \sum_{k=1}^d \pi_k \, \frac{b}{n_k} \sum_{s \in \mathcal{D}_k} \beta^{\mathcal{Z}}_s = \langle \bar{\boldsymbol{\beta}}, \boldsymbol{\pi} \rangle.
\end{equation*}
Evaluating this identity at $\boldsymbol{\pi} = \boldsymbol{e}_k$ identifies $\bar{\beta}_k$ as the expected outcome of a run on domain $k$ alone, which is feasible whenever $n_k \ge b$.

\emph{Variance.}
By the law of total variance and $\mathbb{E}[y \mid S] = \langle \boldsymbol{\beta}^{\mathcal{Z}}, \boldsymbol{1}_S \rangle$,
$\mathrm{Var}(y \mid \boldsymbol{\pi}) = \mathrm{Var}(\langle \boldsymbol{\beta}^{\mathcal{Z}}, \boldsymbol{1}_S \rangle \mid \boldsymbol{\pi}) + \mathbb{E}[\mathrm{Var}(\bar{\eta} \mid S) \mid \boldsymbol{\pi}]$, where the second term is at most $\sigma_\eta^2$.
The domains are sampled independently, and within domain $k$ the sum of $\beta^{\mathcal{Z}}_s$ over a simple random sample of $b \pi_k$ out of $n_k$ sequences has variance $b \pi_k \sigma_k^2 \frac{n_k - b \pi_k}{n_k - 1} \le b \pi_k \sigma_k^2$ by the finite population variance formula~\citep{cochran1977sampling}, where the fraction is read as zero for $n_k = 1$.
Summing over $k$ bounds the first term by $b \sum_k \pi_k \sigma_k^2 \le b \max_k \sigma_k^2$.
\end{proof}

The linear model $f_{\boldsymbol{\beta}}(\boldsymbol{\pi}) = \boldsymbol{\beta}^\top \boldsymbol{\pi}$ fit in Eq.~\eqref{eq:fit_and_optimize} is a datamodel whose units are domains instead of sequences, with coefficients that are budget-scaled averages of sample-level attribution scores.
The partition into domains is arbitrary, since the only assumption is uniform sampling within each domain, and it affects only the noise, which stems from the randomness of training and from which sequences realize the mix, which we call the noise from sequence selection.
Domains that are homogeneous in attribution reduce the latter, but their coefficients take more proxy runs to estimate, as the estimation error grows with $d$ (Lemma~\ref{lem:ols_error}); the finest partition recovers data attribution itself.

\begin{greybox}
\begin{corollary}[\textbf{Domain attribution at the finest granularity is data attribution}]
\label{cor:singleton}
Let every sequence form its own domain, $d = n$, so that a run trained on a subset $S$ of $b$ sequences has the mix $\boldsymbol{\pi} = \boldsymbol{1}_S / b$.
Then $\bar{\boldsymbol{\beta}} = b\, \boldsymbol{\beta}^{\mathcal{Z}}$, so the domain model of Proposition~\ref{prop:domain_attribution} reduces to the benchmark-averaged datamodel of Eq.~\eqref{eq:linear_datamodels}: $\langle \bar{\boldsymbol{\beta}}, \boldsymbol{\pi} \rangle = \langle \boldsymbol{\beta}^{\mathcal{Z}}, \boldsymbol{1}_S \rangle$.
The noise from sequence selection vanishes, leaving $\mathrm{Var}(y \mid \boldsymbol{\pi}) \le \sigma_\eta^2$.
Further, with the equal caps $u_k = 1/b$ ($\kappa = n/b$) that let no sequence repeat within a run, the greedy fill of Eq.~\eqref{eq:vertex_fill} selects a uniformly random subset of $b$ sequences, so vertex sampling coincides with the subsampling design of \citet{ilyas2022datamodels} with fraction $b/n$.
\end{corollary}
\end{greybox}

\begin{proof}
With $n_k = 1$ for every $k$, Proposition~\ref{prop:domain_attribution} applies with $\mu_k = \beta^{\mathcal{Z}}_k$ and $\sigma_k^2 = 0$, and $b \pi_k \in \{0, 1\}$ forces $\boldsymbol{\pi} = \boldsymbol{1}_S / b$.
With caps $u_k = 1/b$, the greedy fill of Eq.~\eqref{eq:vertex_fill} assigns $1/b$ to the first $b$ sequences of a uniformly random permutation, which form a uniformly random subset of size~$b$.
\end{proof}

Domain attribution thus interpolates between data attribution at $d = n$ and a single undifferentiated corpus at $d = 1$, preserving the linear structure in between.
Remark~\ref{rem:banzhaf} relates domain attribution to the data values of cooperative game theory, such as the Shapley and Banzhaf values.

\begin{remark}[\textbf{Relation of domain attribution to Shapley and Banzhaf data values}]\label{rem:banzhaf}
Data valuation scores a sequence $s$ by a weighted average of its marginal contributions $y(S \cup \{s\}; \boldsymbol{z}) - y(S; \boldsymbol{z})$ over subsets $S$, with weights that define the Shapley value~\citep{ghorbani2019data}, the Banzhaf value~\citep{wang2023data}, or the weighted Banzhaf value~\citep{li2023robust}, which includes every other sequence in $S$ independently with probability $\alpha$.
If the linear Datamodel of Eq.~\eqref{eq:linear_datamodels} holds with an intercept for subsets of every size, these values coincide with the attribution score $\beta_s(\boldsymbol{z})$, and by Proposition~\ref{prop:domain_attribution} the domain attribution $\bar{\beta}_k$ is $b$ times the average value of the sequences in domain $k$, up to an additive constant shared by all domains.
Without the linear assumption, the slope coefficients of the affine least squares approximation of $y(S; \boldsymbol{z})$ over all subsets $S$ are the Banzhaf values~\citep{hammer1992approximations}, and under independent inclusion with probability $\alpha$ they are the weighted Banzhaf values with weight $\alpha$~\citep{marichal2011weighted}.
Vertex sampling with equal caps resembles independent inclusion with probability about $1/\kappa$, except that every run contains the same number of domains.
For nonlinear outcomes, its least squares coefficients are therefore a fixed-size analogue of the weighted Banzhaf value.
While a Banzhaf value captures the effect of adding a sequence to a random training set, differences between domain coefficients capture the effect of moving budget between domains at a fixed training size, which is the question that data mixing asks.
\end{remark}

\subsection{The design matrix governs the estimation error}\label{app:estimation-error}

By Proposition~\ref{prop:domain_attribution}, the outcomes of $m$ proxy runs satisfy $\boldsymbol{y} = \boldsymbol{\Pi} \bar{\boldsymbol{\beta}} + \boldsymbol{\xi}$ for a noise vector $\boldsymbol{\xi} \in \mathbb{R}^m$, and the attribution estimate $\hat{\boldsymbol{\beta}} = (\boldsymbol{\Pi}^\top \boldsymbol{\Pi})^{-1} \boldsymbol{\Pi}^\top \boldsymbol{y}$ is the least squares estimator of $\bar{\boldsymbol{\beta}}$.
We make the standard assumptions of linear regression, writing $\boldsymbol{H}_{\boldsymbol{\Pi}} \coloneqq \frac{1}{m} \boldsymbol{\Pi}^\top \boldsymbol{\Pi}$ for the normalized information matrix, i.e., the Gram matrix of the mixes, whose inverse scales the covariance of $\hat{\boldsymbol{\beta}}$.
\begin{greybox}
\begin{assumption}\label{ass:noise}
The noise is unbiased, $\mathbb{E}[\boldsymbol{\xi} \mid \boldsymbol{\Pi}] = \boldsymbol{0}$, its covariance $\boldsymbol{\Sigma}_{\boldsymbol{\Pi}} \coloneqq \mathrm{Cov}(\boldsymbol{\xi} \mid \boldsymbol{\Pi})$ satisfies $v_{\min} \boldsymbol{I}_m \preceq \boldsymbol{\Sigma}_{\boldsymbol{\Pi}} \preceq v_{\max} \boldsymbol{I}_m$ for constants $0 \le v_{\min} \le v_{\max} < \infty$, and the design has full column rank, $\mathrm{rank}(\boldsymbol{\Pi}) = d$.
\end{assumption}
\end{greybox}
Proposition~\ref{prop:domain_attribution} provides the first two conditions for independent proxy runs, where $v_{\max} = \sigma_\eta^2 + b \max_k \sigma_k^2$ is admissible and, for homoscedastic noise $\boldsymbol{\Sigma}_{\boldsymbol{\Pi}} = \sigma^2 \boldsymbol{I}_m$, $v_{\min} = v_{\max} = \sigma^2$, while the third condition is a property of the design that we discuss below.
The simplex is an affine hyperplane that avoids the origin, so $d$ mixes can span $\mathbb{R}^d$, with one direction for the overall outcome level and $d - 1$ for contrasts between domains.

\begin{greybox}
\begin{lemma}[\textbf{Estimation error of linear domain attribution}]
\label{lem:ols_error}
Under Assumption~\ref{ass:noise}, $\hat{\boldsymbol{\beta}}$ is unbiased and its mean squared error satisfies
{\small
\begin{equation}
\label{eq:ols_bound}
v_{\min}\, \mathrm{tr}\big((\boldsymbol{\Pi}^\top \boldsymbol{\Pi})^{-1}\big)
\;\le\;
\mathbb{E}\big[\|\hat{\boldsymbol{\beta}} - \bar{\boldsymbol{\beta}}\|_2^2 \mid \boldsymbol{\Pi}\big]
\;\le\;
v_{\max}\, \mathrm{tr}\big((\boldsymbol{\Pi}^\top \boldsymbol{\Pi})^{-1}\big)
\;\le\;
\frac{d}{m} \cdot \frac{v_{\max}}{\lambda_{\min}(\boldsymbol{H}_{\boldsymbol{\Pi}})},
\end{equation}
}%
with equality in the first two when all runs share the same noise variance, $v_{\min} = v_{\max} = \sigma^2$.
\end{lemma}
\end{greybox}

\begin{proof}
This is the standard variance calculation for least squares.
Substituting $\boldsymbol{y} = \boldsymbol{\Pi} \bar{\boldsymbol{\beta}} + \boldsymbol{\xi}$ gives $\hat{\boldsymbol{\beta}} - \bar{\boldsymbol{\beta}} = (\boldsymbol{\Pi}^\top \boldsymbol{\Pi})^{-1} \boldsymbol{\Pi}^\top \boldsymbol{\xi}$, which has zero conditional mean and expected squared norm $\mathrm{tr}((\boldsymbol{\Pi}^\top \boldsymbol{\Pi})^{-1} \boldsymbol{\Pi}^\top \boldsymbol{\Sigma}_{\boldsymbol{\Pi}} \boldsymbol{\Pi} (\boldsymbol{\Pi}^\top \boldsymbol{\Pi})^{-1})$.
The semidefinite bounds on $\boldsymbol{\Sigma}_{\boldsymbol{\Pi}}$ sandwich this trace between $v_{\min}$ and $v_{\max}$ times $\mathrm{tr}((\boldsymbol{\Pi}^\top \boldsymbol{\Pi})^{-1})$, with equality for $\boldsymbol{\Sigma}_{\boldsymbol{\Pi}} = \sigma^2 \boldsymbol{I}_m$.
The last inequality holds as the $d$ eigenvalues of $(\boldsymbol{\Pi}^\top \boldsymbol{\Pi})^{-1} = \frac{1}{m} \boldsymbol{H}_{\boldsymbol{\Pi}}^{-1}$ are at most $1 / (m \lambda_{\min}(\boldsymbol{H}_{\boldsymbol{\Pi}}))$.
\end{proof}

The last inequality bounds the error by the smallest eigenvalue of~$\boldsymbol{H}_{\boldsymbol{\Pi}}$, which is small when some combination of domains barely varies across the proxy runs, i.e., when the runs leave a direction of the simplex unexplored.
In optimal experimental design~\citep{kiefer1959regression,pukelsheim2006optimal}, minimizing the trace in Eq.~\eqref{eq:ols_bound}, maximizing this eigenvalue, and maximizing $\det(\boldsymbol{H}_{\boldsymbol{\Pi}})$ are known as the A-, E-, and D-criteria.
This eigenvalue is what the condition number $\mathrm{cond}(\boldsymbol{\Pi}) = \sqrt{\lambda_{\max}(\boldsymbol{H}_{\boldsymbol{\Pi}}) / \lambda_{\min}(\boldsymbol{H}_{\boldsymbol{\Pi}})}$ controls: every mix sums to one, so $\boldsymbol{\Pi} \boldsymbol{1}_d = \boldsymbol{1}_m$ and $\boldsymbol{1}_d^\top \boldsymbol{H}_{\boldsymbol{\Pi}} \boldsymbol{1}_d = 1$, hence $\lambda_{\max}(\boldsymbol{H}_{\boldsymbol{\Pi}}) \ge 1/d$ and $1 / \lambda_{\min}(\boldsymbol{H}_{\boldsymbol{\Pi}}) \le d\, \mathrm{cond}(\boldsymbol{\Pi})^2$.
Minimizing the condition number of $\boldsymbol{\Pi}$, as in Section~\ref{sec:experiments-tuning}, therefore controls the last bound of Eq.~\eqref{eq:ols_bound} up to the factor $d$.
The noise level and the number of domains are fixed by the data and its partition, so the number of runs $m$ and the design $\boldsymbol{\Pi}$ are the only two levers of a regression-based data mixing method.
We next compare the designs of Dirichlet and vertex sampling on these quantities.

\subsection{Pitfalls of Dirichlet sampling}\label{app:dirichlet-theory}

Dirichlet sampling in Eq.~\eqref{eq:dirichlet_mixing} draws each mix as $\boldsymbol{\pi}_i \sim \mathrm{Dirichlet}(\gamma_i\, \bar{\boldsymbol{a}})$ around the tempered prior $\bar{\boldsymbol{a}} \coloneqq \boldsymbol{\pi}_{\text{prior}}^{1/\tau} / \|\boldsymbol{\pi}_{\text{prior}}^{1/\tau}\|_1$, with a log-uniform concentration $\gamma_i \in [\gamma_{\mathrm{lo}}, \gamma_{\mathrm{hi}}]$.
We analyze it without the caps and the rounding of $b \pi_k$ to integers, which keeps the moments in closed form, and start with the expected information matrix $\bar{\boldsymbol{H}} = \mathbb{E}[\boldsymbol{H}_{\boldsymbol{\Pi}}]$.

\begin{greybox}
\begin{lemma}[\textbf{Information matrix of Dirichlet sampling}]
\label{lem:dirichlet_H}
Under Eq.~\eqref{eq:dirichlet_mixing},
\begin{tightequation}
\bar{\boldsymbol{H}} \coloneqq \mathbb{E}[\boldsymbol{\pi}_i \boldsymbol{\pi}_i^\top] = c \operatorname{diag}(\bar{\boldsymbol{a}}) + (1 - c)\, \bar{\boldsymbol{a}} \bar{\boldsymbol{a}}^\top,
\quad
c \coloneqq
1 - \frac{\ln\big((1 + \gamma_{\mathrm{hi}}) / (1 + \gamma_{\mathrm{lo}})\big)}{\ln(\gamma_{\mathrm{hi}} / \gamma_{\mathrm{lo}})} \in (0, 1).
\label{eq:dirichlet_H}
\end{tightequation}%
Consequently, $\mathrm{tr}(\bar{\boldsymbol{H}}^{-1}) = \frac{1}{c} \big( \sum_{k=1}^d \bar{a}_k^{-1} - (1 - c)\, d \big)$ and $c\, \bar{a}_{\min} \le \lambda_{\min}(\bar{\boldsymbol{H}}) \le c\, \bar{a}_{(2)}$, where $\bar{a}_{\min}$ and $\bar{a}_{(2)}$ denote the smallest and the second smallest entry of $\bar{\boldsymbol{a}}$.
\end{lemma}
\end{greybox}

\begin{proof}
\emph{Second moment.}
For $\boldsymbol{\pi} \sim \mathrm{Dirichlet}(\boldsymbol{\alpha})$ with $\alpha_0 \coloneqq \sum_k \alpha_k$, the moments $\mathbb{E}[\pi_j \pi_k] = \frac{\alpha_j \alpha_k}{\alpha_0 (\alpha_0 + 1)}$ for $j \ne k$ and $\mathbb{E}[\pi_k^2] = \frac{\alpha_k (\alpha_k + 1)}{\alpha_0 (\alpha_0 + 1)}$ read in matrix form as $\mathbb{E}[\boldsymbol{\pi} \boldsymbol{\pi}^\top] = \frac{\boldsymbol{\alpha} \boldsymbol{\alpha}^\top + \operatorname{diag}(\boldsymbol{\alpha})}{\alpha_0 (\alpha_0 + 1)}$.
With $\boldsymbol{\alpha} = \gamma \bar{\boldsymbol{a}}$ and $\alpha_0 = \gamma$, since $\bar{\boldsymbol{a}} \in \Delta^{d-1}$,
\begin{equation*}
\mathbb{E}[\boldsymbol{\pi} \boldsymbol{\pi}^\top \mid \gamma] = \frac{1}{1 + \gamma} \operatorname{diag}(\bar{\boldsymbol{a}}) + \frac{\gamma}{1 + \gamma}\, \bar{\boldsymbol{a}} \bar{\boldsymbol{a}}^\top.
\end{equation*}
The two weights sum to one, so averaging over $\gamma$ gives Eq.~\eqref{eq:dirichlet_H} with $c = \mathbb{E}_\gamma[1/(1 + \gamma)] \in (0, 1)$.
For $\log \gamma$ uniform on $[\log \gamma_{\mathrm{lo}}, \log \gamma_{\mathrm{hi}}]$, substituting $\gamma = e^t$ and using $\frac{\mathrm{d}}{\mathrm{d}t} [t - \ln(1 + e^t)] = \frac{1}{1 + e^t}$ yields the closed form
\begin{equation*}
c = \frac{1}{\ln(\gamma_{\mathrm{hi}} / \gamma_{\mathrm{lo}})} \int_{\ln \gamma_{\mathrm{lo}}}^{\ln \gamma_{\mathrm{hi}}} \frac{\mathrm{d}t}{1 + e^t} = \frac{\big[ t - \ln(1 + e^t) \big]_{\ln \gamma_{\mathrm{lo}}}^{\ln \gamma_{\mathrm{hi}}}}{\ln(\gamma_{\mathrm{hi}} / \gamma_{\mathrm{lo}})} = 1 - \frac{\ln\big((1 + \gamma_{\mathrm{hi}}) / (1 + \gamma_{\mathrm{lo}})\big)}{\ln(\gamma_{\mathrm{hi}} / \gamma_{\mathrm{lo}})}.
\end{equation*}

\emph{Inverse.}
Let $\boldsymbol{D} \coloneqq \operatorname{diag}(\bar{\boldsymbol{a}})$, which is positive definite as $\bar{a}_k > 0$, so $\bar{\boldsymbol{H}} \succeq c \boldsymbol{D} \succ 0$.
The Sherman--Morrison formula~\citep{sherman1950adjustment} applied to the rank-one update $\bar{\boldsymbol{H}} = c \boldsymbol{D} + (1 - c)\, \bar{\boldsymbol{a}} \bar{\boldsymbol{a}}^\top$, together with $\boldsymbol{D}^{-1} \bar{\boldsymbol{a}} = \boldsymbol{1}_d$ and $\bar{\boldsymbol{a}}^\top \boldsymbol{1}_d = 1$, gives
\begin{equation*}
\bar{\boldsymbol{H}}^{-1} = \frac{1}{c} \boldsymbol{D}^{-1} - \frac{(1 - c)\, c^{-2}\, \boldsymbol{1}_d \boldsymbol{1}_d^\top}{1 + (1 - c) / c} = \frac{1}{c} \boldsymbol{D}^{-1} - \frac{1 - c}{c}\, \boldsymbol{1}_d \boldsymbol{1}_d^\top,
\end{equation*}
and taking the trace yields $\mathrm{tr}(\bar{\boldsymbol{H}}^{-1}) = \frac{1}{c} (\sum_k \bar{a}_k^{-1} - (1 - c)\, d)$.

\emph{Spectrum.}
The lower bound follows from $\bar{\boldsymbol{H}} \succeq c \boldsymbol{D}$.
For the upper bound, let $k_1, k_2$ index the two smallest entries of $\bar{\boldsymbol{a}}$ and pick a unit vector $\boldsymbol{w} \in \mathrm{span}(\boldsymbol{e}_{k_1}, \boldsymbol{e}_{k_2})$ orthogonal to $\bar{\boldsymbol{a}}$, which exists as orthogonality is a single linear constraint on a two-dimensional space.
Then $\boldsymbol{w}^\top \bar{\boldsymbol{H}} \boldsymbol{w} = c\, (\bar{a}_{k_1} w_{k_1}^2 + \bar{a}_{k_2} w_{k_2}^2) \le c\, \bar{a}_{(2)}$, and the Rayleigh principle gives $\lambda_{\min}(\bar{\boldsymbol{H}}) \le c\, \bar{a}_{(2)}$.
\end{proof}

\begin{figure}[t]
    \centering
    \includegraphics[width=0.5\linewidth]{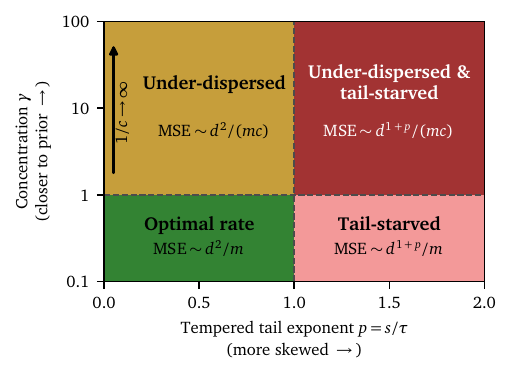}
    \caption{
    \textbf{Two pitfalls of Dirichlet sampling.}
    Order of the estimation error in Theorem~\ref{thm:dirichlet_design} as a function of the Dirichlet concentration $\gamma$ and the tempered tail exponent $p = s/\tau$ of the Zipfian prior.
    The concentration is fixed, so that $c = 1/(1 + \gamma)$, and the boundaries at $\gamma = 1$ and $p = 1$ are schematic.
    For $\gamma \gg 1$, the mixes cluster around the prior, and the error grows with $1/c \approx \gamma$.
    For $p > 1$, the rarest domains receive too little mass, and the error grows super-quadratically in $d$.
    }
    \label{fig:dirichlet_phase_transitions}
\end{figure}

By Lemma~\ref{lem:dirichlet_H}, $\bar{\boldsymbol{H}}$ averages two designs with weights $c$ and $1 - c$: single-domain runs at $\boldsymbol{e}_k$ drawn with probability $\bar{a}_k$, and the tempered prior repeated in every run.
These two terms lead to two pitfalls of Dirichlet sampling (Figure~\ref{fig:dirichlet_phase_transitions}).
(i) \emph{Under-dispersion}: a large concentration keeps the mixes close to the tempered prior, so $c$ is small and the information in every direction orthogonal to $\bar{\boldsymbol{a}}$ shrinks by the factor $c$.
(ii) \emph{Tail starvation}: even at $c \rightarrow 1$, rare domains with a small share $\bar{a}_k$ are seldom explored, so their attributions are estimated poorly and dominate the error, as $\mathrm{tr}(\bar{\boldsymbol{H}}^{-1})$ in Lemma~\ref{lem:dirichlet_H} grows with $\sum_k \bar{a}_k^{-1}$.

To quantify the second pitfall, we model the prior mix as Zipfian~\citep{zipf1949human}, $(\boldsymbol{\pi}_{\text{prior}})_k \propto k^{-s}$ for domains ranked by size (Figure~\ref{fig:prior_zipf_extended}).
Tempering preserves this form, as $(k^{-s})^{1/\tau} = k^{-p}$, so the tempered prior is $\bar{a}_k = k^{-p} / \zeta_d(p)$ with the flatter exponent $p \coloneqq s / \tau$ and the truncated zeta function $\zeta_d(p) \coloneqq \sum_{k=1}^d k^{-p}$.
Its rarest domain has the share $\bar{a}_{\min} = d^{-p} / \zeta_d(p)$, which shrinks as $d$ grows, so by Lemma~\ref{lem:dirichlet_H}, $\lambda_{\min}(\bar{\boldsymbol{H}}) \ge c\, d^{-p} / \zeta_d(p)$.
To bound the error of a sampled design, which depends on $\boldsymbol{H}_{\boldsymbol{\Pi}}$ rather than on $\bar{\boldsymbol{H}}$, we use the following auxiliary results.

\begin{figure}[t]
    \centering
    \includegraphics[width=\linewidth]{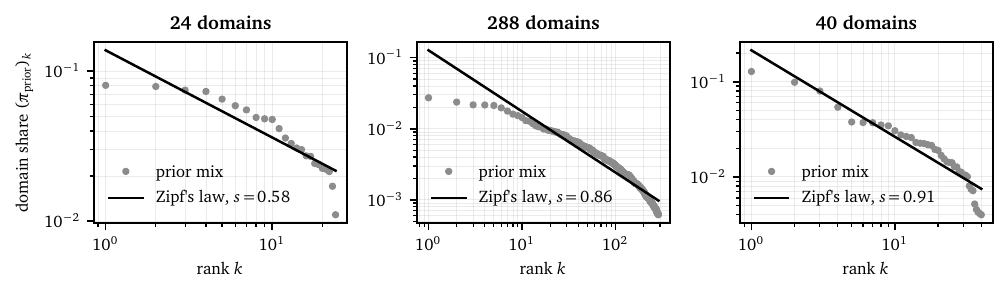}
    \caption{
    \textbf{Prior data mixes are approximately Zipfian.}
    Domain share against rank on log-log axes for the three setups, with least-squares fits of Zipf's law $(\boldsymbol{\pi}_{\text{prior}})_k \propto k^{-s}$ in log-log coordinates.
    The 288-domain setup retains the larger half of the $576$ topic--format pairs, which truncates its tail.
    }
    \label{fig:prior_zipf_extended}
\end{figure}

\begin{greybox}
\begin{lemma}[\textbf{Jensen's inequality for the inverse}~\citep{groves1969note}]\label{lem:jensen_inverse}
Let $\boldsymbol{X}$ be a random positive definite matrix with $\bar{\boldsymbol{X}} \coloneqq \mathbb{E}[\boldsymbol{X}]$ positive definite.
Then $\mathbb{E}[\boldsymbol{w}^\top \boldsymbol{X}^{-1} \boldsymbol{w}] \ge \boldsymbol{w}^\top \bar{\boldsymbol{X}}^{-1} \boldsymbol{w}$ for every vector $\boldsymbol{w}$, and in particular $\mathbb{E}[\mathrm{tr}(\boldsymbol{X}^{-1})] \ge \mathrm{tr}(\bar{\boldsymbol{X}}^{-1})$.
\end{lemma}
\end{greybox}
\begin{proof}
Expanding $(\boldsymbol{X}^{-1} \boldsymbol{w} - \bar{\boldsymbol{X}}^{-1} \boldsymbol{w})^\top \boldsymbol{X}\, (\boldsymbol{X}^{-1} \boldsymbol{w} - \bar{\boldsymbol{X}}^{-1} \boldsymbol{w}) \ge 0$ gives $\boldsymbol{w}^\top \boldsymbol{X}^{-1} \boldsymbol{w} \ge 2\, \boldsymbol{w}^\top \bar{\boldsymbol{X}}^{-1} \boldsymbol{w} - \boldsymbol{w}^\top \bar{\boldsymbol{X}}^{-1} \boldsymbol{X} \bar{\boldsymbol{X}}^{-1} \boldsymbol{w}$, whose expectation is $\boldsymbol{w}^\top \bar{\boldsymbol{X}}^{-1} \boldsymbol{w}$; summing over $\boldsymbol{w} = \boldsymbol{e}_1, \ldots, \boldsymbol{e}_d$ gives the trace inequality.
\end{proof}

\begin{greybox}
\begin{lemma}[\textbf{Matrix Chernoff bound, lower tail}]\label{lem:matrix_chernoff}
Let $\boldsymbol{X}_1, \ldots, \boldsymbol{X}_m$ be independent random positive semidefinite $d \times d$ matrices with $\lambda_{\max}(\boldsymbol{X}_i) \le R$ almost surely, and let $\mu_{\min} \coloneqq \lambda_{\min}(\sum_{i=1}^m \mathbb{E}[\boldsymbol{X}_i])$.
Then, for every $t \in [0, 1)$,
\begin{equation*}
\Pr\Big[ \lambda_{\min}\Big( \sum_{i=1}^m \boldsymbol{X}_i \Big) \le (1 - t)\, \mu_{\min} \Big] \le d \Big[ \frac{e^{-t}}{(1 - t)^{1 - t}} \Big]^{\mu_{\min} / R} \le d\, e^{-t^2 \mu_{\min} / (2 R)}.
\end{equation*}
\end{lemma}
\end{greybox}
\begin{proof}
The first inequality is Theorem~1.1 of \citet{tropp2012user}.
The second one follows from $(1 - t) \ln(1 - t) \ge t^2/2 - t$ for $t \in [0, 1)$, which holds as both sides vanish at $t = 0$ and the derivative of their difference, $-\ln(1 - t) - t$, is nonnegative.
\end{proof}

\begin{greybox}
\begin{corollary}[\textbf{Concentration of the information matrix}]\label{cor:concentration}
Let the rows of $\boldsymbol{\Pi}$ be independent with $\|\boldsymbol{\pi}_i\|_2^2 \le R$ almost surely and a common expected information matrix $\bar{\boldsymbol{H}} \coloneqq \mathbb{E}[\boldsymbol{\pi}_i \boldsymbol{\pi}_i^\top] \succ 0$, and let the noise satisfy the first two conditions of Assumption~\ref{ass:noise}.
For any $\delta \in (0, 1)$, if $m \ge 8 R \ln(d / \delta) / \lambda_{\min}(\bar{\boldsymbol{H}})$, then with probability at least $1 - \delta$ over the draw of $\boldsymbol{\Pi}$,
\begin{equation}
\lambda_{\min}(\boldsymbol{H}_{\boldsymbol{\Pi}}) \ge \tfrac{1}{2} \lambda_{\min}(\bar{\boldsymbol{H}})
\qquad \text{and} \qquad
\mathbb{E}\big[\|\hat{\boldsymbol{\beta}} - \bar{\boldsymbol{\beta}}\|_2^2 \mid \boldsymbol{\Pi}\big] \le \frac{2\, d\, v_{\max}}{m\, \lambda_{\min}(\bar{\boldsymbol{H}})}.
\label{eq:concentration}
\end{equation}
\end{corollary}
\end{greybox}
\begin{proof}
The summands $\boldsymbol{X}_i \coloneqq \boldsymbol{\pi}_i \boldsymbol{\pi}_i^\top$ of $m \boldsymbol{H}_{\boldsymbol{\Pi}}$ are independent and positive semidefinite with $\mathbb{E}[\boldsymbol{X}_i] = \bar{\boldsymbol{H}}$ and $\lambda_{\max}(\boldsymbol{X}_i) = \|\boldsymbol{\pi}_i\|_2^2 \le R$, so Lemma~\ref{lem:matrix_chernoff} with $\mu_{\min} = m \lambda_{\min}(\bar{\boldsymbol{H}})$ and $t = 1/2$ bounds the probability of $\lambda_{\min}(\boldsymbol{H}_{\boldsymbol{\Pi}}) \le \frac{1}{2} \lambda_{\min}(\bar{\boldsymbol{H}})$ by $d\, e^{-m \lambda_{\min}(\bar{\boldsymbol{H}}) / (8 R)} \le \delta$, using the assumed lower bound on $m$.
On the complementary event, $\lambda_{\min}(\boldsymbol{H}_{\boldsymbol{\Pi}}) > 0$, so $\boldsymbol{\Pi}$ has full column rank and the last bound of Eq.~\eqref{eq:ols_bound} gives the error bound.
\end{proof}

\begin{greybox}
\begin{theorem}[\textbf{Estimation error of Dirichlet sampling under a Zipfian prior}]
\label{thm:dirichlet_design}
Let the rows of $\boldsymbol{\Pi}$ be drawn independently by Eq.~\eqref{eq:dirichlet_mixing} with a Zipfian prior of exponent $s$, let the noise satisfy the first two conditions of Assumption~\ref{ass:noise}, and let $p = s/\tau$ and $c$ be as above.
\begin{enumerate}
\item[(i)] For any $\delta \in (0, 1)$, if $m \ge 8\, d^{p} \zeta_d(p) \ln(d/\delta) / c$, then with probability at least $1 - \delta$ over the draw of $\boldsymbol{\Pi}$,
\begin{equation}
\mathbb{E}\big[\|\hat{\boldsymbol{\beta}} - \bar{\boldsymbol{\beta}}\|_2^2 \mid \boldsymbol{\Pi}\big] \le \frac{2\, v_{\max}}{m\, c}\, d^{1+p} \zeta_d(p).
\label{eq:dirichlet_upper}
\end{equation}
\item[(ii)] For any $m \ge d$, in expectation over the draw of $\boldsymbol{\Pi}$,
\begin{equation}
\mathbb{E}\big[\|\hat{\boldsymbol{\beta}} - \bar{\boldsymbol{\beta}}\|_2^2\big] \ge \frac{v_{\min}}{m\, c} \Big( \frac{d^{1+p} \zeta_d(p)}{1 + p} - (1 - c)\, d \Big).
\label{eq:dirichlet_lower}
\end{equation}
\end{enumerate}
For $d \ge 2 (1 + p)$ and homoscedastic noise, both bounds are within constant factors depending only on $p$ of $\frac{\sigma^2}{m c}\, d^{1+p} \zeta_d(p)$, which for fixed $p$ grows as
\begin{equation}
\frac{d^{1+p} \zeta_d(p)}{m\, c} = \Theta\Bigg( \frac{d^2}{m\, c} \cdot
\begin{cases}
1, & p < 1,\\
\ln d, & p = 1,\\
d^{\,p - 1}, & p > 1.
\end{cases} \Bigg)
\label{eq:dirichlet_rate}
\end{equation}
\end{theorem}
\end{greybox}

\begin{proof}
Recall that $\bar{a}_{\min} = d^{-p} / \zeta_d(p)$ and $\lambda_{\min}(\bar{\boldsymbol{H}}) \ge c\, d^{-p} / \zeta_d(p)$.

\emph{Part (i).}
Every mix satisfies $\|\boldsymbol{\pi}_i\|_2^2 \le \|\boldsymbol{\pi}_i\|_1^2 = 1$, so Corollary~\ref{cor:concentration} applies with $R = 1$.
Its condition on $m$ is implied by the assumed one, as $8 \ln(d/\delta) / \lambda_{\min}(\bar{\boldsymbol{H}}) \le 8\, d^{p} \zeta_d(p) \ln(d/\delta) / c$, and its error bound reads $\frac{2 d\, v_{\max}}{m \lambda_{\min}(\bar{\boldsymbol{H}})} \le \frac{2\, v_{\max}}{m c}\, d^{1+p} \zeta_d(p)$.

\emph{Part (ii).}
For $m \ge d$, the design has full column rank almost surely, because the Dirichlet distribution has a density on the simplex, which does not contain the origin, so any $d$ of its draws are almost surely linearly independent.
Hence Eq.~\eqref{eq:ols_bound} applies almost surely, and Lemma~\ref{lem:jensen_inverse} with $\boldsymbol{X} = \boldsymbol{\Pi}^\top \boldsymbol{\Pi}$ and $\bar{\boldsymbol{X}} = m \bar{\boldsymbol{H}}$ gives
\begin{equation*}
\mathbb{E}\big[\|\hat{\boldsymbol{\beta}} - \bar{\boldsymbol{\beta}}\|_2^2\big] \ge v_{\min}\, \mathbb{E}\big[\mathrm{tr}\big((\boldsymbol{\Pi}^\top \boldsymbol{\Pi})^{-1}\big)\big] \ge \frac{v_{\min}}{m}\, \mathrm{tr}(\bar{\boldsymbol{H}}^{-1}) = \frac{v_{\min}}{m c} \Big( \zeta_d(p) \sum_{k=1}^d k^{p} - (1 - c)\, d \Big),
\end{equation*}
using the trace formula of Lemma~\ref{lem:dirichlet_H} with $\bar{a}_k^{-1} = k^{p} \zeta_d(p)$.
The claim follows from $\sum_{k=1}^d k^{p} \ge \int_0^d x^{p} \, \mathrm{d}x = \frac{d^{1+p}}{1 + p}$.

\emph{Rate.}
Since $\zeta_d(p) \ge d \cdot d^{-p}$, we have $d^{1+p} \zeta_d(p) \ge d^2$, so for $d \ge 2 (1 + p)$ the subtracted term satisfies $(1 - c)\, d \le d \le \frac{d^{1+p} \zeta_d(p)}{2 (1 + p)}$, and the right-hand side of Eq.~\eqref{eq:dirichlet_lower} is at least $\frac{v_{\min}}{m c} \cdot \frac{d^{1+p} \zeta_d(p)}{2 (1 + p)}$.
Together with Eq.~\eqref{eq:dirichlet_upper}, both bounds are within the factors $\frac{1}{2(1+p)}$ and $2$ of $\frac{\sigma^2}{m c}\, d^{1+p} \zeta_d(p)$ for $v_{\min} = v_{\max} = \sigma^2$.
Finally, comparing the sum $\zeta_d(p)$ with the integral of $x^{-p}$ gives $\zeta_d(p) = \frac{d^{1-p}}{1 - p} (1 + o(1))$ for $p < 1$, $\zeta_d(p) = \ln d + O(1)$ for $p = 1$, and $\zeta_d(p) = \zeta(p) (1 + o(1))$ for $p > 1$ with the Riemann zeta function $\zeta$, which yields Eq.~\eqref{eq:dirichlet_rate}.
\end{proof}

Theorem~\ref{thm:dirichlet_design} shows that, compared to the optimal error $\sigma^2 d^2 / m$ on the simplex (Proposition~\ref{prop:vertex_design} below), under-dispersion increases the error of Dirichlet sampling by a factor of order $1/c$, and tail starvation by a further factor $d^{p - 1}$ when $p > 1$, i.e., when the prior is too skewed for the temperature $\tau$ to flatten (Figure~\ref{fig:dirichlet_phase_transitions}).
Both factors also appear in the lower bound of part (ii), so the two bounds match up to constants.
For the uniform prior of Theorem~\ref{thm:capped_vertex}, $p = 0$ and only the factor $1/c$ remains.
In our experiments, tail starvation is negligible: the Zipf fits in Figure~\ref{fig:prior_zipf_extended} give $s \approx 0.6$ at 24 domains and $s \approx 0.9$ at 288 and 40 domains, so with $\tau = 2$ all three setups lie in the regime $p < 1$.
Under-dispersion, however, is forced by the caps.
Without caps, a smaller concentration always helps, as it increases $c$, and the range $\gamma \in [0.1, 10]$ used at 24 domains gives $c = 0.5$.
At 288 and 40 domains, small concentrations produce sparse mixes that violate the caps and are almost all rejected, so we tune the ranges with the criterion of Section~\ref{sec:experiments-tuning} to $\gamma \in [40, 4000]$ and $[100, 10^4]$, for which $c \approx 0.005$ and $0.002$.
By Lemma~\ref{lem:dirichlet_H} and the proof of part (ii), the expected error of the uncapped Dirichlet design is then at least $\mathrm{tr}(\bar{\boldsymbol{H}}^{-1}) / d^2 \approx 2$, $230$, and $540$ times the optimum $\sigma^2 d^2 / m$ for the empirical priors and $\gamma$ ranges of our three setups.
Part of this gap is unavoidable, as under equal caps every design on $\mathcal{P}_{\text{cap}}$ pays a factor of about $d / (\kappa - 1)$ (Theorem~\ref{thm:capped_vertex}).
Under the equal-cap approximation, Dirichlet sampling thus remains about $4$ and $30$ times worse than the best capped design at 288 and 40 domains.
This matches Figure~\ref{fig:plot_budget_main}, where both samplers fit well at 24 domains but no surrogate fits the Dirichlet swarm at 288, and Section~\ref{sec:experiments_proxy_metrics}, where near-prior Dirichlet mixes vary HellaSwag BPB five times less than vertex mixes.

\subsection{On optimal sampling}\label{app:vertex-theory}

Without caps, the vertices of the simplex are the pure-domain mixes $\boldsymbol{e}_1, \ldots, \boldsymbol{e}_d$, and the greedy fill of Eq.~\eqref{eq:vertex_fill} places each run at a uniformly random vertex.
Allocating exactly $m/d$ runs to each vertex gives the balanced vertex design, the classical optimal design for first-degree mixture models~\citep{scheffe1958experiments,kiefer1961optimum}.
We restate it because its error $\sigma^2 d^2 / m$ is the reference optimum in Appendix~\ref{app:dirichlet-theory} and Theorem~\ref{thm:capped_vertex} builds on its proof.

\begin{greybox}
\begin{proposition}[\textbf{Classical optimality of the balanced vertex design on the simplex}]
\label{prop:vertex_design}
Let $m$ be a multiple of $d$, let Assumption~\ref{ass:noise} hold, and let $\boldsymbol{\Pi}$ be the balanced vertex design that places $m/d$ runs at each pure-domain mix $\boldsymbol{e}_k$.
Then $\boldsymbol{\Pi}^\top \boldsymbol{\Pi} = \frac{m}{d} \boldsymbol{I}_d$, each $\hat{\beta}_k$ is the average outcome of the runs on domain $k$, and
\begin{equation}
\mathbb{E}\big[\|\hat{\boldsymbol{\beta}} - \bar{\boldsymbol{\beta}}\|_2^2 \mid \boldsymbol{\Pi}\big] \le \frac{d^2}{m}\, v_{\max},
\label{eq:vertex_mse}
\end{equation}
with equality for homoscedastic noise.
Among all designs with $m$ rows in $\Delta^{d-1}$, the balanced vertex design is, up to the order of the runs, the unique design that minimizes $\mathrm{tr}((\boldsymbol{\Pi}^\top \boldsymbol{\Pi})^{-1})$, maximizes $\det(\boldsymbol{\Pi}^\top \boldsymbol{\Pi})$, maximizes $\lambda_{\min}(\boldsymbol{\Pi}^\top \boldsymbol{\Pi})$, and minimizes the worst-case prediction variance $\max_{\boldsymbol{\pi} \in \Delta^{d-1}} \boldsymbol{\pi}^\top (\boldsymbol{\Pi}^\top \boldsymbol{\Pi})^{-1} \boldsymbol{\pi}$.
\end{proposition}
\end{greybox}

\begin{proof}
\emph{Balanced vertex design.}
Each row of $\boldsymbol{\Pi}$ is a standard basis vector, so $(\boldsymbol{\Pi}^\top \boldsymbol{\Pi})_{kl}$ counts the runs placed at $\boldsymbol{e}_k$ if $k = l$ and vanishes otherwise, giving $\boldsymbol{\Pi}^\top \boldsymbol{\Pi} = \frac{m}{d} \boldsymbol{I}_d$.
Then $\hat{\boldsymbol{\beta}} = \frac{d}{m} \boldsymbol{\Pi}^\top \boldsymbol{y}$ averages the outcomes of the runs on each domain, and $\mathrm{tr}((\boldsymbol{\Pi}^\top \boldsymbol{\Pi})^{-1}) = d^2 / m$, so Eq.~\eqref{eq:vertex_mse} and its equality case follow from Lemma~\ref{lem:ols_error}.

\emph{Optimality.}
Let $\boldsymbol{\Pi}$ be any design with rows $\boldsymbol{\pi}_i \in \Delta^{d-1}$ and $\boldsymbol{M} \coloneqq \boldsymbol{\Pi}^\top \boldsymbol{\Pi}$ positive definite with eigenvalues $\lambda_1, \ldots, \lambda_d > 0$; designs with a singular $\boldsymbol{M}$ are never optimal, as all four quantities are then degenerate.
The diagonal of $\boldsymbol{M}$ sums to
\begin{equation}
\mathrm{tr}(\boldsymbol{M}) = \sum_{i=1}^m \|\boldsymbol{\pi}_i\|_2^2 \le \sum_{i=1}^m \|\boldsymbol{\pi}_i\|_1^2 = m,
\label{eq:trace_bound}
\end{equation}
with equality if and only if every row has a single nonzero entry, i.e., is a vertex.
For the trace of the inverse, the Cauchy--Schwarz inequality $(\sum_j \lambda_j)(\sum_j \lambda_j^{-1}) \ge d^2$ and Eq.~\eqref{eq:trace_bound} give $\mathrm{tr}(\boldsymbol{M}^{-1}) \ge d^2 / \mathrm{tr}(\boldsymbol{M}) \ge d^2 / m$, with equality if and only if all eigenvalues are equal and $\mathrm{tr}(\boldsymbol{M}) = m$, i.e., $\boldsymbol{M} = \frac{m}{d} \boldsymbol{I}_d$.
For the determinant, the inequality of arithmetic and geometric means gives $\det(\boldsymbol{M}) = \prod_j \lambda_j \le (\mathrm{tr}(\boldsymbol{M}) / d)^d \le (m / d)^d$, with the same equality condition.
For the smallest eigenvalue, $\lambda_{\min}(\boldsymbol{M}) \le \mathrm{tr}(\boldsymbol{M}) / d \le m / d$, again with the same equality condition.
For the worst-case prediction variance, $\boldsymbol{\pi} \mapsto \boldsymbol{\pi}^\top \boldsymbol{M}^{-1} \boldsymbol{\pi}$ is convex, so its maximum over the simplex is attained at a vertex, and $\max_k (\boldsymbol{M}^{-1})_{kk} \ge \frac{1}{d} \mathrm{tr}(\boldsymbol{M}^{-1}) \ge d / m$, with equality if and only if $\boldsymbol{M}^{-1}$ has a constant diagonal and $\mathrm{tr}(\boldsymbol{M}^{-1}) = d^2 / m$, which again forces $\boldsymbol{M} = \frac{m}{d} \boldsymbol{I}_d$.
Finally, $\boldsymbol{M} = \frac{m}{d} \boldsymbol{I}_d$ requires every row to be a vertex by the equality case of Eq.~\eqref{eq:trace_bound}, and then $\boldsymbol{M} = \operatorname{diag}(m_1, \ldots, m_d)$ with $m_k$ runs at $\boldsymbol{e}_k$ forces $m_k = m / d$, i.e., the balanced vertex design up to the order of its rows, which attains all four~bounds.
\end{proof}

Proposition~\ref{prop:vertex_design} formalizes why vertex sampling works: measuring one domain at a time makes the estimates $\hat{\beta}_k$ uncorrelated for independent runs, and measuring every domain equally often gives each the same variance, at most $v_{\max} d / m$.
The balanced vertex design is optimal under all classical criteria at once (A, D, E, and the worst-case prediction variance;~\citealp{pukelsheim2006optimal}), which in general favor different designs, so the choice of criterion does not matter here.
Since $\boldsymbol{\pi}_{\text{target}}$ is chosen by optimizing the prediction $\langle \hat{\boldsymbol{\beta}}, \boldsymbol{\pi} \rangle$ over mixes, the last criterion matters most for data mixing: the prediction of every mix $\boldsymbol{\pi}$ has variance at most $v_{\max} d / m$, which no other design can improve in the worst case.
In practice, the caps are smaller than one, so a run cannot train on a single domain, and \methodx samples the vertices of $\mathcal{P}_{\text{cap}}$, each of which fills several domains up to their caps.
Theorem~\ref{thm:capped_vertex} shows that the expected information matrix of vertex sampling is optimal on the capped simplex for a uniform prior.

\begin{greybox}
\begin{theorem}[\textbf{Vertex sampling is optimal on the capped simplex with equal caps}]
\label{thm:capped_vertex}
Let the prior be uniform, so that all caps equal $u_k = \kappa / d$ with $q \coloneqq d / \kappa \in \{1, \ldots, d - 1\}$ an integer, and let the runs be drawn independently.
\begin{enumerate}
\item[(i)] The greedy fill of Eq.~\eqref{eq:vertex_fill} places each run at the cap on a uniformly random subset of $q$ domains, and its expected information matrix $\bar{\boldsymbol{V}} \coloneqq \mathbb{E}[\boldsymbol{\pi}_i \boldsymbol{\pi}_i^\top]$ has the eigenvalue $1/d$ along $\boldsymbol{1}_d$ and the eigenvalue $\frac{\kappa - 1}{d (d - 1)}$ with multiplicity $d - 1$ on its orthogonal complement, so that $\mathrm{tr}(\bar{\boldsymbol{V}}^{-1}) = d + \frac{d (d-1)^2}{\kappa - 1}$.
\item[(ii)] Every positive definite matrix $\boldsymbol{H}$ in the convex hull of $\{\boldsymbol{\pi} \boldsymbol{\pi}^\top : \boldsymbol{\pi} \in \mathcal{P}_{\text{cap}}\}$, which includes $\boldsymbol{H}_{\boldsymbol{\Pi}}$ for every design with rows in $\mathcal{P}_{\text{cap}}$ and the expected information matrix of every random such design, satisfies
{\small
\begin{equation}
\lambda_{\min}(\boldsymbol{H}) \le \frac{\kappa - 1}{d (d - 1)},
\qquad
\mathrm{tr}(\boldsymbol{H}^{-1}) \ge d + \frac{d (d - 1)^2}{\kappa - 1},
\qquad
\det(\boldsymbol{H}) \le \frac{1}{d} \Big( \frac{\kappa - 1}{d (d - 1)} \Big)^{d - 1},
\label{eq:capped_optimality}
\end{equation}
}
with equality for $\boldsymbol{H} = \bar{\boldsymbol{V}}$, so that the expected design of vertex sampling minimizes $\mathrm{tr}(\boldsymbol{H}^{-1})$ and maximizes $\lambda_{\min}(\boldsymbol{H})$ and $\det(\boldsymbol{H})$ over $\mathcal{P}_{\text{cap}}$.
\item[(iii)] Under the noise conditions of Assumption~\ref{ass:noise}, for any $\delta \in (0, 1)$, if $m \ge \frac{8 \kappa (d - 1)}{\kappa - 1} \ln(d / \delta)$, then with probability at least $1 - \delta$ over the draw of $\boldsymbol{\Pi}$, $\lambda_{\min}(\boldsymbol{H}_{\boldsymbol{\Pi}}) \ge \frac{\kappa - 1}{2 d (d - 1)}$ and
\begin{equation}
\mathbb{E}\big[\|\hat{\boldsymbol{\beta}} - \bar{\boldsymbol{\beta}}\|_2^2 \mid \boldsymbol{\Pi}\big] \le \frac{2\, d^2 (d - 1)}{m\, (\kappa - 1)}\, v_{\max}.
\label{eq:capped_mse}
\end{equation}
\end{enumerate}
\end{theorem}
\end{greybox}

\begin{proof}
\emph{Part (i).}
With equal caps $u \coloneqq \kappa / d = 1 / q$, the greedy fill assigns the cap to the first $q$ domains of the permutation $\sigma_i$, which exhausts the budget as $q u = 1$, and zero to the remaining domains.
The first $q$ elements of a uniformly random permutation form a uniformly random $q$-subset $S_i$, so $\boldsymbol{\pi}_i = u\, \boldsymbol{1}_{S_i}$ with $\Pr[k \in S_i] = q / d$ and $\Pr[k, l \in S_i] = \frac{q (q - 1)}{d (d - 1)}$ for $k \ne l$.
Hence
\begin{equation*}
\bar{\boldsymbol{V}} = u^2 \Big[ \Big( \frac{q}{d} - \frac{q (q - 1)}{d (d - 1)} \Big) \boldsymbol{I}_d + \frac{q (q - 1)}{d (d - 1)}\, \boldsymbol{1}_d \boldsymbol{1}_d^\top \Big],
\end{equation*}
whose eigenvalue along $\boldsymbol{1}_d$ is $u^2 \big( \frac{q}{d} + \frac{q (q - 1)}{d} \big) = u^2 q^2 / d = 1 / d$, and whose eigenvalue on the orthogonal complement is $u^2\, \frac{q}{d} \cdot \frac{d - q}{d - 1} = \frac{d - q}{q\, d\, (d - 1)} = \frac{\kappa - 1}{d (d - 1)}$, using $u^2 q = 1 / q$ and $q = d / \kappa$.
The trace of $\bar{\boldsymbol{V}}^{-1}$ follows by summing the reciprocals with multiplicities $1$ and $d - 1$.

\emph{Part (ii).}
Let $\boldsymbol{H} = \sum_j w_j \boldsymbol{\pi}_j \boldsymbol{\pi}_j^\top$ be positive definite with weights $w_j \ge 0$ summing to one and $\boldsymbol{\pi}_j \in \mathcal{P}_{\text{cap}}$.
Two constraints hold for every such $\boldsymbol{H}$.
First, every mix sums to one, so $\boldsymbol{1}_d^\top \boldsymbol{H} \boldsymbol{1}_d = \sum_j w_j (\boldsymbol{\pi}_j^\top \boldsymbol{1}_d)^2 = 1$, i.e., the unit vector $\boldsymbol{v} \coloneqq \boldsymbol{1}_d / \sqrt{d}$ satisfies $\boldsymbol{v}^\top \boldsymbol{H} \boldsymbol{v} = 1/d$.
Second, for $\boldsymbol{\pi} \in \mathcal{P}_{\text{cap}}$, $\|\boldsymbol{\pi}\|_2^2 = \sum_k \pi_k^2 \le \sum_k u\, \pi_k = u$, with equality if and only if every $\pi_k \in \{0, u\}$, i.e., $\boldsymbol{\pi}$ is a vertex, so $\mathrm{tr}(\boldsymbol{H}) \le u = \kappa / d$.
Let $\boldsymbol{p}_1, \ldots, \boldsymbol{p}_{d-1}$ be an orthonormal basis of the orthogonal complement of $\boldsymbol{v}$, collected in $\boldsymbol{P} \in \mathbb{R}^{d \times (d-1)}$, so that $\boldsymbol{Q} \coloneqq [\boldsymbol{v}\ \boldsymbol{P}]$ is orthogonal, and let $\boldsymbol{C} \coloneqq \boldsymbol{P}^\top \boldsymbol{H} \boldsymbol{P}$ be the lower-right block of $\boldsymbol{Q}^\top \boldsymbol{H} \boldsymbol{Q}$.
The two constraints give $\mathrm{tr}(\boldsymbol{C}) = \mathrm{tr}(\boldsymbol{H}) - \boldsymbol{v}^\top \boldsymbol{H} \boldsymbol{v} \le (\kappa - 1) / d$.
For the smallest eigenvalue, restricting the Rayleigh quotient of $\boldsymbol{H}$ to the range of $\boldsymbol{P}$ gives $\lambda_{\min}(\boldsymbol{H}) \le \lambda_{\min}(\boldsymbol{C}) \le \mathrm{tr}(\boldsymbol{C}) / (d - 1) \le \frac{\kappa - 1}{d (d - 1)}$.
For the determinant, Fischer's inequality for the positive definite block matrix $\boldsymbol{Q}^\top \boldsymbol{H} \boldsymbol{Q}$~\citep[Section~7.8]{horn2013matrix} and the inequality of arithmetic and geometric means give $\det(\boldsymbol{H}) = \det(\boldsymbol{Q}^\top \boldsymbol{H} \boldsymbol{Q}) \le (\boldsymbol{v}^\top \boldsymbol{H} \boldsymbol{v}) \det(\boldsymbol{C}) \le \frac{1}{d} \big( \frac{\mathrm{tr}(\boldsymbol{C})}{d - 1} \big)^{d - 1} \le \frac{1}{d} \big( \frac{\kappa - 1}{d (d - 1)} \big)^{d - 1}$.
For the trace of the inverse, $\mathrm{tr}(\boldsymbol{H}^{-1}) = \mathrm{tr}(\boldsymbol{Q}^\top \boldsymbol{H}^{-1} \boldsymbol{Q}) = \boldsymbol{v}^\top \boldsymbol{H}^{-1} \boldsymbol{v} + \sum_{j=1}^{d-1} \boldsymbol{p}_j^\top \boldsymbol{H}^{-1} \boldsymbol{p}_j$.
For any unit vector $\boldsymbol{x}$, the Cauchy--Schwarz inequality $1 = (\boldsymbol{x}^\top \boldsymbol{H}^{1/2} \boldsymbol{H}^{-1/2} \boldsymbol{x})^2 \le (\boldsymbol{x}^\top \boldsymbol{H} \boldsymbol{x})(\boldsymbol{x}^\top \boldsymbol{H}^{-1} \boldsymbol{x})$ gives $\boldsymbol{x}^\top \boldsymbol{H}^{-1} \boldsymbol{x} \ge 1 / (\boldsymbol{x}^\top \boldsymbol{H} \boldsymbol{x})$.
Applied to $\boldsymbol{v}$, this gives $\boldsymbol{v}^\top \boldsymbol{H}^{-1} \boldsymbol{v} \ge d$; applied to each $\boldsymbol{p}_j$ and combined with $(\sum_j x_j)(\sum_j 1/x_j) \ge (d-1)^2$ for positive $x_j$, it gives $\sum_j \boldsymbol{p}_j^\top \boldsymbol{H}^{-1} \boldsymbol{p}_j \ge \sum_j 1 / (\boldsymbol{p}_j^\top \boldsymbol{H} \boldsymbol{p}_j) \ge (d - 1)^2 / \mathrm{tr}(\boldsymbol{C}) \ge \frac{d (d - 1)^2}{\kappa - 1}$.
Finally, $\bar{\boldsymbol{V}}$ lies in the convex hull as an expectation, and by part (i) it has $\boldsymbol{v}$ as an eigenvector with eigenvalue $1/d$ and $\boldsymbol{C} = \frac{\kappa - 1}{d (d - 1)} \boldsymbol{I}_{d - 1}$, so all three bounds hold with equality.

\emph{Part (iii).}
Every vertex has $\|\boldsymbol{\pi}_i\|_2^2 = q u^2 = u$, so Corollary~\ref{cor:concentration} applies with $R = u = \kappa / d$ and $\lambda_{\min}(\bar{\boldsymbol{V}}) = \frac{\kappa - 1}{d (d - 1)}$ from part (i), for which its condition on $m$ and its error bound are exactly those claimed.
\end{proof}

Theorem~\ref{thm:capped_vertex} extends Proposition~\ref{prop:vertex_design} to the capped simplex: without caps, $q = 1$ and part (ii) reduces to the bounds of the proposition.
The caps force each run to spread over $q = d / \kappa$ domains, so a run reveals only about $1/q$ as much about any single domain.
This is why $\mathrm{tr}(\bar{\boldsymbol{V}}^{-1})$ in part (i) exceeds the uncapped value $d^2$ by a factor of about $d / (\kappa - 1)$, which is $4.5$, $62$, and $16$ in our three setups if we treat their caps as equal.
Vertex sampling is nevertheless optimal for two reasons.
First, every mix sums to one, so all designs carry the same information about the overall outcome level ($1/d$ along $\boldsymbol{1}_d$) and differ only in the information about contrasts between domains.
Vertices attain the largest norm $\|\boldsymbol{\pi}\|_2^2 = \kappa / d$ on $\mathcal{P}_{\text{cap}}$, so a design of vertices has the largest total information $\mathrm{tr}(\boldsymbol{H})$, which leaves vertex sampling the most information for these contrasts; mixes that are not vertices have a strictly smaller norm and are never optimal.
Second, it spreads this information evenly across these contrasts, which is what the three quantities in Eq.~\eqref{eq:capped_optimality} reward.
Section~\ref{sec:method} calls the design near-orthogonal in expectation because the expected off-diagonal entries of $\boldsymbol{\Pi}^\top \boldsymbol{\Pi}$ are only $(q - 1) / (d - 1) \approx 1 / \kappa$ times the diagonal ones.
Part (iii) shows that a sampled design is nearly optimal once $m$ exceeds a multiple of $d \ln d$: its estimation error is at most $2 d / (d - 1) \approx 2$ times the smallest value any design on $\mathcal{P}_{\text{cap}}$ can achieve under homoscedastic noise.
This guarantee is conservative, as its constants come from the matrix Chernoff bound and its condition on $m$ exceeds the size of our swarms; in practice, we select the best-conditioned of many sampled designs (Appendix~\ref{app:experimental-setup}), which controls the realized design directly but is not covered by the theorem.
Pairing runs by reversed permutations, as in Section~\ref{sec:method}, does not affect the theorem.
The reverse of a uniform permutation is uniform, so $\bar{\boldsymbol{V}}$ stays the same, and for $\kappa \ge 2$ the two runs of a pair have disjoint supports, so part (iii) still holds with the $m/2$ independent pairs as summands, each with expectation $2 \bar{\boldsymbol{V}}$ and $R = u$.
We leave non-uniform priors with unequal caps to future work, as they break the symmetry that gives $\bar{\boldsymbol{V}}$ its closed form.

\section{Limitations}\label{app:limitations}

A limitation shared by all regression-based methods is the cost of training a proxy swarm, which we quantify here by scaling proxy model size and budget.
In practice, the cost-to-performance trade-off of data mixing becomes only more favorable as the target model size and pretraining budget increase.
For example, pretraining a 4B model on 120B tokens for 2.5k GPU hours could be designed with a swarm of $512\times$ 30M models computed for 0.1k GPU hours (4\% of the target cost; cf. Appendix~\ref{app:compute_resources}).
Naturally, we train the target models using a single random repetition to enable more experimental setups.
Our theory bounds the error of the estimated attribution $\hat{\boldsymbol{\beta}}$, not the error of the predictions made with it.
Moreover, the noise level $v_{\max}$ in our bounds depends on how the corpus is split into domains, so the $d^2/m$ rate alone cannot tell how many domains to use.
Finally, whether the optimal mix is dense in positively attributed domains depends on the target benchmark, corpus size, and pretraining token budget.
We thus view the curvature $r$ that controls this as a useful utility trade-off, while the hyperparameter-free property of \methodx applies to the design of proxy-mix sampling and regression model fitting.

\section{Compute Resources}\label{app:compute_resources}

All models are trained on NVIDIA GH200 96GB GPUs without special optimizations beyond the default code implementations.
Training one OLMo2 proxy model takes 0.2 GPU-hours for 30M parameters (1B tokens), 0.7 for 60M (2B), and 1.7 for 100M (3.2B).
The Qwen3.5-like 100M proxy takes 2.6 GPU-hours on the same 3.2B tokens, 1.5 times as much as OLMo2, which we attribute to its hybrid attention with gated DeltaNet layers that add constant overhead independent of model size.
Model evaluation adds further cost per proxy, depending on the benchmark and its sample size.
The proxy swarms cost 0.7k GPU-hours at 30M (1024 models at 24 domains and 2048 at 288 domains), 1.4k at 60M (2048 models), 1.8k at 100M for OLMo2 (1024 models), and 2.7k at 100M for Qwen3.5 (1024 models).
Pretraining a single 1B target model for 32B tokens takes 233 GPU-hours using 16 GPUs, so the 23 runs behind Table~\ref{tab:performance1b} and the 7 runs behind Tables~\ref{tab:ablation_scaling} and~\ref{tab:ablation_r} amount to ca. 7.0k.
The 1B Qwen3.5-like target model takes 272 GPU-hours on the same 32B tokens, so the four runs behind the CPile column of Table~\ref{tab:performance4b} amount to ca. 1.1k.
One 4B target run (80B tokens) takes 1.7k GPU-hours using 32 GPUs, so the 10 runs behind Table~\ref{tab:performance4b} amount to ca. 17k.
Together with model evaluations, the reported experiments cost ca. 34k GPU-hours; we estimate that the whole project consumed ca. 51k GPU-hours, including preliminary work and failed jobs.

\end{document}